%% file: main.tex
\documentclass{article}
\usepackage{microtype}
\usepackage{graphicx}
\usepackage{booktabs} %
\usepackage{hyperref}
\usepackage{url}
\usepackage{natbib}
\usepackage{wrapfig}
\usepackage{subcaption}

\usepackage{hyperref}

\usepackage[accepted]{icml2024}

\usepackage{amsmath,bm}
\usepackage{amssymb}
\usepackage{mathtools}
\usepackage{amsthm}
\usepackage{thmtools}
\usepackage{thm-restate}
\usepackage{enumitem}
\usepackage{algorithmic}
\usepackage{algorithm}
\usepackage[algo2e]{algorithm2e}
\renewcommand{\algorithmicrequire}{\textbf{Input:}}

\input{utils}

\usepackage{tikz-cd}
\usepackage{ctable}
\usetikzlibrary{shapes.geometric}
\usepackage{nicefrac}

\usepackage[page,header]{appendix}
\usepackage{titletoc}
\usepackage[most]{tcolorbox}

\author{}

\date{}

\begin{document}
\twocolumn[
\icmltitle{The Privacy Power of Correlated Noise in Decentralized Learning}

\begin{icmlauthorlist}
\icmlauthor{Youssef Allouah}{yyy}
\icmlauthor{Anastasia Koloskova}{yyy}
\icmlauthor{Aymane El Firdoussi}{yyy}
\icmlauthor{Martin Jaggi}{yyy}
\icmlauthor{Rachid Guerraoui}{yyy}
\end{icmlauthorlist}

\icmlaffiliation{yyy}{EPFL, Switzerland}

\icmlcorrespondingauthor{Youssef Allouah}{youssef.allouah@epfl.ch}

\icmlkeywords{Machine Learning, ICML, differential privacy, decentralized learning, optimization}

\vskip 0.3in
]

\printAffiliationsAndNotice{} %
\input{_abstract}

\input{_intro}

\input{problem}

\input{algorithm}

\input{theory}

\input{experiments}

\input{conclusion}

\bibliography{references}
\bibliographystyle{icml2024}

\newpage
\onecolumn
\appendix
\input{appendix.tex}

\end{document}

%% file: utils.tex
\newcommand{\alg}{\textsc{Decor}\xspace}

\providecommand{\lin}[1]{\ensuremath{\left\langle #1 \right\rangle}}

\providecommand{\norm}[1]{\left\lVert#1\right\rVert}
\providecommand{\ac}[1]{a{\left(#1\right)}}
\providecommand{\hg}[1]{\mathrm{H}_\cG{\left(#1\right)}}
\providecommand{\loss}{\mathcal{L}}

\newcommand{\clip}[1]{\mathrm{Clip}{(#1)}}

  \providecommand{\R}{\mathbb{R}} %
  \providecommand{\N}{\mathbb{N}} %
  
  \DeclareMathOperator{\E}{{\mathbb E}}
  \providecommand{\EE}[2]{{\mathbb E}_{#1}\left.#2\right. }  %

  \providecommand{\0}{\mathbf{0}}
  \providecommand{\1}{\mathbf{1}}
  \renewcommand{\aa}{\mathbf{a}}
  \providecommand{\bb}{\mathbf{b}}

  \providecommand{\ee}{\mathbf{e}}

  \renewcommand{\gg}{\mathbf{g}}

  \providecommand{\uu}{\mathbf{u}}
  \providecommand{\vv}{\mathbf{v}}
  
  \providecommand{\xx}{\mathbf{x}}
  \providecommand{\xxi}{\mathbf{\xi}}
  \providecommand{\yy}{\mathbf{y}}

  \providecommand{\mA}{\mathbf{A}}

  \providecommand{\mI}{\mathbf{I}}
  
  \providecommand{\mK}{\mathbf{K}}
  \providecommand{\mL}{\mathbf{L}}
  \providecommand{\mM}{\mathbf{M}}
  \providecommand{\mN}{\mathbf{N}}

  \providecommand{\mS}{\mathbf{S}}

  \providecommand{\mW}{\mathbf{W}}
  \providecommand{\mX}{\mathbf{X}}

  \providecommand{\cA}{\mathcal{A}}

  \providecommand{\cD}{\mathcal{D}}
  \providecommand{\cE}{\mathcal{E}}
  
  \providecommand{\cG}{\mathcal{G}}
  \providecommand{\cH}{\mathcal{H}}
  \providecommand{\cI}{\mathcal{I}}

  \providecommand{\cM}{\mathcal{M}}
  \providecommand{\cN}{\mathcal{N}}

  \providecommand{\cS}{\mathcal{S}}

  \providecommand{\cX}{\mathcal{X}}
  \providecommand{\cY}{\mathcal{Y}}

  \providecommand{\SSigma}{\mathbf{\Sigma}}

  \newcommand{\card}[1]{\left|#1\right|}
  \newcommand{\sigmacor}{\sigma_{\mathrm{cor}}}
  \newcommand{\sigmacdp}{\sigma_{\mathrm{cdp}}}
  
  \usepackage{bm}

  \usepackage[textwidth=5cm]{todonotes}
  
\providecommand{\mycomment}[3]{\todo[caption={},size=footnotesize,color=#1!20]{\textbf{#2: }#3}}%
\providecommand{\inlinecomment}[3]{%
  {\color{#1}#2: #3}}%
\newcommand\commenter[2]%
{%
  \expandafter\newcommand\csname i#1\endcsname[1]{\inlinecomment{#2}{#1}{##1}}
  \expandafter\newcommand\csname #1\endcsname[1]{\mycomment{#2}{#1}{##1}}
}

\newtheorem{lemma}{Lemma}
\newtheorem{corollary}[lemma]{Corollary}

\newtheorem{definition}{Definition}

\newtheorem{assumption}{Assumption}
\newtheorem{theorem}[lemma]{Theorem}

%% file: _abstract.tex
\begin{abstract}
Decentralized learning is appealing as it enables the scalable usage of large amounts of distributed data and resources (without resorting to any central entity), while promoting privacy since every user minimizes the direct exposure of their data. Yet,  without additional precautions, curious users can still leverage models obtained from their peers to violate privacy. In this paper, we propose \alg, a variant of decentralized SGD with differential privacy (DP) guarantees. In \alg, users securely exchange randomness seeds in one communication round to generate pairwise-canceling correlated Gaussian noises, which are injected to protect local models at every communication round. We theoretically and empirically show that, for arbitrary connected graphs, \alg matches the central DP optimal privacy-utility trade-off. We do so under SecLDP, our new relaxation of local DP, which protects all user communications against an external eavesdropper and curious users, assuming that every pair of connected users shares a secret, i.e., an information hidden to all others. The main theoretical challenge is to control the accumulation of non-canceling correlated noise due to network sparsity. We also propose a companion SecLDP privacy accountant for public use.

\end{abstract}

%% file: _intro.tex
\section{Introduction}
\label{sec:intro}
In numerous machine learning scenarios, the training dataset is dispersed among diverse sources, including individual users or distinct organizations responsible for generating each data segment. The nature of such data often involves privacy concerns, especially in applications like healthcare~\cite{sheller2020federated}, which can divulge sensitive information about an individual's health. 
Privacy issues make it either impractical or undesirable to transfer the data beyond their original sources, promoting the 
emergence of \emph{federated} and \emph{decentralized} learning~\cite{mcmahan2017communication,lian2017can}, 
where the training occurs directly on the data-holding entities.
Decentralized learning additionally removes the assumption of a central server, with only the model updates being transmitted directly between users.
A classical decentralized learning algorithm is decentralized stochastic gradient descent (D-SGD)~\cite{koloskova2020unified}, where users alternate between performing local gradient updates and averaging local models via gossiping.

When dealing with privacy-sensitive data, it is crucial not 
only to confine the sensitive information locally with decentralization, but also to 
ensure that the algorithm avoids leaking any sensitive information through its communicated updates or the final model.
These can be observed by an \emph{external eavesdropper} or even an \emph{honest-but-curious user}, who follows the algorithm but may attempt to violate the privacy of other users.
The notion of \emph{differential privacy} (DP)~\citep{dwork2014algorithmic} serves as a widely accepted theoretical framework for measuring  formal privacy guarantees.
This notion has been extensively studied in centralized settings~\cite{bassily2014private,abadi2016deep}, i.e., assuming a trusted data curator or server.
Yet, much less attention has been given to adapting DP to decentralized learning.

Several threat models have been considered in decentralized learning, the strongest corresponding to the classical notion of \emph{local differential privacy} (LDP)~\cite{kasiviswanathan2011can}.
Under LDP, users do not trust any other entity and obfuscate all their communications independently.
In contrast, \emph{central differential privacy} (CDP) only protects the final model, exactly as if the learning was conducted on a single machine.
Importantly, there is a significant gap in performance between LDP and CDP algorithms. In a system of $n$ users, the optimal privacy-utility trade-off
under LDP can be $n$ times worse than CDP~\cite{duchi2018minimax}.
Indeed, the CDP baseline is the variant of D-SGD adding noise to protect the average of user models only, which is much less noise than that needed under LDP, to protect every local model before averaging.
Some prior works aimed at reconciling this performance gap by investigating other relaxations of LDP.
For example, in federated learning with an untrusted server, the shuffle model~\cite{cheu2019distributed} and distributed DP~\cite{kairouz2021distributed} restrict the view of the server using cryptographic primitives and match the CDP optimal privacy-utility trade-off.
However, these approaches are server-based and thus cannot be used in decentralized learning.
Network DP~\cite{cyffers2022privacy}  considers honest-but-curious users whose view is restricted to their neighboring communications.
As we discuss in Section~\ref{sec:rel-work}, the %
privacy-utility trade-offs under Network DP match CDP only for well-connected graphs~\cite{cyffers2022muffliato}.

\textbf{Contributions.}
We propose \alg, a new algorithm for decentralized learning with differential privacy. 
\alg is a variant of D-SGD, which additionally injects two types of privacy noise to protect local models: (i) uncorrelated Gaussian noise to protect the local model \emph{after} gossip averaging, and (ii) correlated Gaussian noise, as a sum of pairwise cancelling noise terms for each neighbor, to protect local models \emph{before} gossip averaging.
In the presence of a server, after one round of \alg, averaging all local models would cancel out the correlated Gaussian noise terms, and leave the uncorrelated Gaussian noise protecting the average of models, as was previously studied by~\citet{sabater2022accurate} (see Section~\ref{sec:rel-work}).
However, on a sparse graph, the correlated noise terms do not all cancel out in \alg. 
To obtain our main result, we control the accumulation of correlated noise across iterations in our convergence analysis, and show that its effect vanishes across iterations of \alg. 

We consider an external eavesdropper and honest-but-curious non-colluding users and 
show that \alg matches the optimal CDP privacy-utility trade-off under our new relaxation of LDP we call \emph{secret-based local differential privacy} (SecLDP). Our relaxation protects against an external eavesdropper and curious users who can observe all communications, assuming that every pair of connected users shares a secret, i.e., an information a priori hidden to all others, similar to secure aggregation~\cite{bonawitz2017practical}.
For example, we consider the secrets to be shared randomness seeds exchangeable in one round of encrypted communications.
Following the choice of the set of secrets, our relaxation can capture several threat models, e.g., including collusion of several users; or recovering LDP when no communications are secret.

We also demonstrate the empirical superiority of \alg over the LDP baseline on simulated and real-world data and multiple network topologies, and provide a practical SecLDP privacy accountant for \alg.\footnote{Our code is available at \href{https://github.com/elfirdoussilab1/DECOR}{https://github.com/elfirdoussilab1/DECOR}.}

\input{related}

%% file: related.tex
\subsection{Related Work}
\label{sec:rel-work}
Most works on DP optimization have focused on the centralized setting~\citep{chaudhuri2011differentially,bassily2014private,abadi2016deep}, where a trusted curator collects user data.
Also, several recent works tackle privacy in federated learning, where an honest-but-curious server coordinates the users. These works either use cryptographic primitives to only reveal the sum of updates~\citep{jayaraman2018distributed,kairouz2021distributed,agarwal2021skellam} or to anonymize user identities through shuffling~\citep{erlingsson2019amplification,cheu2019distributed}.
Although these techniques provably achieve the centralized optimal privacy-utility trade-off, they are incompatible with fully decentralized settings, where only peer-to-peer communications are allowed, or induce large computational and communication costs. 

\textbf{Private decentralized learning.} In decentralized settings, several distributed optimization algorithms~\citep{bellet2018personalized,cheng2019towards,huang2019dp,li2023convergence} have been adapted, by adding noise to gradient updates, to ensure LDP. However, these approaches yield a poor privacy-utility trade-off, which is a fundamental drawback of LDP~\citep{duchi2013local}.
\citet{cyffers2022privacy} consider a weaker privacy model than LDP where the threat comes from curious users solely, who can observe information exchanged with their communication graph neighbors only.
Under this weaker privacy threat, it is possible to match the centralized privacy-utility trade-off for well-connected graphs only~\citep{cyffers2022muffliato}. 
In general, SecLDP and Network DP are orthogonal, since the latter restricts the view of users to local communications only, while the former hides part of the global communications---secrets---to an adversary observing all other communications.
Yet, when considering honest-but-curious users, SecLDP is arguably stronger than Network DP as users in SecLDP have a larger view, i.e., all communications besides secrets outside their neighborhood. 
Also, the privacy-utility trade-off achieved under SecLDP is matches CDP for arbitrary connected topologies, unlike Network DP.

\textbf{Correlated noise.} Our correlated noise technique has been studied in various forms within secure multi-party computation, where the goal is to privately compute a function without a trusted central entity.
A first form, called secret sharing~\citep{shamir1979share}, consists in adding uniformly random noise terms which cancel out only if enough users collude. The same idea has also been analyzed for decentralized averaging~\citep{li2019privacy}.
However, these works guarantee the perfect security of the inputs, not the privacy of the average.
Indeed, a curious adversary observing the average can infer the presence of an input or reconstruct it~\citep{melis2019exploiting}.
In this direction, \citet{imtiaz2019distributed} proposed adding correlated Gaussian noise to the inputs, along with a smaller uncorrelated Gaussian noise to protect the average only.
The correlated Gaussian noise is generated by having users sample Gaussian noise locally, and using secure aggregation~\citep{bonawitz2017practical} to get the average of the noise terms, which is subtracted by users.
Thus, averaging privatized inputs cancels out correlated noises and only leaves the smaller uncorrelated noise to protect the average.
However, the algorithm requires a central entity for secure aggregation, which is not possible in decentralized learning and can be costly in communication.
\citet{sabater2022accurate} further adapted the correlated Gaussian noise technique to decentralized settings, without using secure aggregation, by having connected users exchange pairwise cancelling Gaussian noise. However, their work only studies decentralized averaging, and does not cover the more challenging decentralized learning scenario, where the non-cancelled correlated noise accumulates across training iterations.
Finally, we remark that correlated noise has also been studied in centralized settings with a different meaning, e.g., correlation is across iterations~\cite{kairouz2021practical}, which is orthogonal to our work where noise is correlated across the users, but is uncorrelated across the iterations.

%% file: problem.tex
\section{Problem Statement}
\label{sec:problem}
We consider a set of users $[n] \coloneqq \{1,\ldots,n\}$ who want to collaboratively solve a common machine learning task in a decentralized fashion. Each user $i \in [n]$ holds a local dataset $\cD_i$ containing $m \in \N$ elements $\{\xxi_i^1,\dots,\xxi_i^{m}\}$ from data space $\cX$.\footnote{All datasets have the same size for simplicity; our theory can be directly extended to cover local datasets with different sizes.}
The goal is to minimize the following global loss function:
\begin{align}
 \min_{\xx \in \R^d}  \loss(\xx):=\frac{1}{n} \sum_{i=1}^n \loss_i(\xx), \label{eq:f}
\end{align} 
where the local loss functions $\loss_i \colon \R^d \to \R, i \in [n],$ are distributed among~$n$ users and are given in empirical form:
\begin{align}\label{eq:F_i}
 \loss_i(\xx) := \frac{1}{\card{\cD_i}} \sum_{\xxi \in \cD_i} \ell(\xx,\xxi),\quad \forall \xx \in \R^d,
\end{align}
where $\ell(\xx, \xxi) \in \R$ is the loss of parameter $\xx$ on sample $\xxi$.
We study the fully decentralized setting where users are the nodes of an undirected communication graph $\cG = ([n], \cE)$. Two nodes $i, j \in [n]$ can communicate directly if they are neighbors in $\cG$, i.e., $\{i, j\} \in \cE$.

\textbf{Secret-based local DP.}
We aim to protect the privacy of user data against an adversary who can \textit{eavesdrop} on all communications, while every pair of connected users $\{i,j\} \in \cE$ shares a sequence of \emph{secrets} $\mS_{ij}$,
which represent observations of random variables commonly known to the nodes sharing the secrets only. In practice, these are locally generated via
\emph{shared randomness seeds} exchanged after one round of encrypted communications~\cite{bonawitz2017practical}, and conceptually one can consider the secrets to be the shared randomness seeds only.
We denote by $\cS_{\mathrm{all}} \coloneqq \left\{\mS_{ij} \colon \{i,j\} \in \cE\right\}$ the set of all secrets. 
While local DP (LDP)~\cite{kasiviswanathan2011can} protects the privacy of all communications without assuming the existence of secrets, at the price of a poor privacy-utility trade-off~\cite{duchi2013local}, we propose to relax LDP into \emph{secret-based local differential privacy} (SecLDP) as defined below.

\begin{definition}[SecLDP]
    \label{def:ldp}
    Let $\varepsilon \geq 0$, $\delta \in [0, 1]$. Consider a randomized decentralized algorithm $\cA : \cX^{m \times n} \to \cY$, which outputs the transcript of all communications. Algorithm $\cA$ satisfies $(\varepsilon,\delta, \mathcal{S})$-SecLDP if it satisfies $(\varepsilon, \delta)$-DP given that the set of secrets $\mathcal{S}$ is unknown to the adversary. That is, for every adjacent datasets $\cD, \cD' \in \cX^{m \times n}$,
    \begin{equation*}
        \mathbb{P} \left[\cA{(\cD)} ~\middle|~ \cS~\text{is hidden}\right] \leq e^{\varepsilon} \cdot \mathbb{P} \left[ \cA{(\cD')} ~\middle|~ \cS~\text{is hidden} \right] + \delta,
    \end{equation*}
    where the event ``$\cS$ is hidden'' conditions on the non-secret observations $\cS_{\mathrm{all}} \setminus \cS$.
    We say that $\cA$ satisfies $(\varepsilon,\delta)$-SecLDP if it satisfies $(\varepsilon,\delta, \mathcal{S})$-SecLDP and $\mathcal{S}$ is clear from the context.
\end{definition}
Our privacy definition can encode several levels of knowledge of the adversary, and the corresponding threat models, through the choice of the secrets $\cS$. Essentially, the larger the set of secrets, the weaker is the adversary.
To see this, we denote by $\cS_i \coloneqq \left\{\mS_{jk} \colon \{j,k\} \in \cE~\text{and}~j,k \neq i \right\}$ the set of secrets hidden from user $i \in [n]$, and by $\cS_\cI \coloneqq \cap_{i \in \cI} \cS_i$ the set of secrets hidden from the group of users $\cI \subseteq [n]$, so that $\cS_\cI \subseteq \cS_i \subseteq \cS_{\mathrm{all}}$ for every $i \in \cI \subseteq [n]$. 
We consider the following adversaries in increasing strength:
 \begin{enumerate}[label=\textnormal{\Roman*.}]
    \item External eavesdropper: the only adversary is not a user and ignores all the secrets $\cS_{\mathrm{all}}$, but can eavesdrop on all communications between users. This threat is covered by $(\varepsilon, \delta, \cS_{\mathrm{all}})$-SecLDP.
     \item Honest-but-curious users without collusion: every user faithfully follows the protocol, but may try to infer private information from other users by eavesdropping on all communications, while knowing the secrets it shares with other users only. This threat is covered by having $(\varepsilon, \delta, \cS_i)$-SecLDP for every $i \in [n]$.
     \item Honest-but-curious users with partial collusion: every group of users of size $q < n$ may collude by disclosing the secrets they have access to. This threat is covered, at collusion level $q$, by having $(\varepsilon, \delta, \cS_\cI)$-SecLDP for every $\cI \subseteq [n], \card{\cI}= q$.
     \item Full collusion: all users may collude against any other user in the system, as if the adversary can observe all communications and no secrets are hidden from them. This threat is covered by $(\varepsilon, \delta, \varnothing)$-SecLDP, which corresponds to LDP.
 \end{enumerate}
The adversaries above are in increasing strength in the sense that defending against adversary II consequently defends against adversary I, and the same logic holds for the other adversaries.
In this work, we consider the secrets to be shared randomness seeds, which allow every pair of users to keep the same observation of a random variable and generate correlated noise. In practice, such secrets, i.e., randomness seeds, can be shared securely and efficiently, as is common in secure aggregation for federated learning~\cite{bonawitz2017practical,kairouz2021distributed}.
Moreover, for ease of exposition, we focus on the adversaries of type I and II above and defer the extension of our results to types III and IV to the appendix.

\textbf{Comparison with other relaxations.}
Recall from Section~\ref{sec:intro} that a common relaxation of LDP is central differential privacy (CDP), where the adversary can only access the final training model. In fact, CDP is recovered from SecLDP by considering the larger set of secrets consisting of all user communications. From a privacy point of view, CDP is equivalent to DP in the trusted curator model---the privacy model in the centralized setting---and thus allows achieving the best privacy-utility trade-off.
In contrast, the best achievable mean squared error under LDP is $n$ times worse than under CDP~\citep{duchi2018minimax,allouah2023privacy}, for strongly convex optimization problems.
Indeed, the CDP baseline is D-SGD with additional Gaussian noise magnitude $\Theta{(\tfrac{1}{n\varepsilon^2})}$, while the LDP baseline D-SGD with Gaussian noise magnitude $\Theta{(\tfrac{1}{\varepsilon^2})}$. We refer to these approaches as the CDP and LDP baselines, respectively.

However, CDP does not protect against honest-but-curious users, who can be expected in real-world scenarios. This limitation motivated Network DP~\cite{cyffers2022privacy}, which guarantees the privacy of all communications against honest-but-curious users whose view is restricted to communications with their neighbors, with privacy-utility trade-offs sometimes matching those of CDP~\cite{cyffers2022muffliato}. 
In general, SecLDP and Network DP are orthogonal, since the latter restricts the communications known to users, while the former restricts part of these communications---secrets---to an adversary observing all other communications.
In the case of the aforementioned adversary II, SecLDP is arguably stronger than Network DP because honest-but-curious users in SecLDP have a larger view, i.e., all communications besides secrets outside their neighborhood.

%% file: algorithm.tex
\section{\alg: Decentralized SGD with Correlated Noise}
\label{sec:algorithm}

We now present our algorithm \alg, summarized in Algorithm~\ref{algo}. Overall, \alg is a variant of D-SGD injecting the privacy noise each local model. %
This privacy noise consists of two parts: (i) correlated noise to protect the local communications before gossip averaging, and (ii) uncorrelated noise to protect the gossip average.

\begin{algorithm}[t!]
    \caption{\textsc{\alg: decentralized SGD with correlated noise}}\label{algo}
	\begin{algorithmic}[1]
		\REQUIRE for each user $i\in [n]$ initialize $\xx_i^{(0)} \in \R^d$, 
		 stepsizes $\{\eta_t\}_{t=0}^{T-1}$, number of iterations $T$, 
        clipping threshold $C$,
        noise parameters $\sigmacor$ and $\sigmacdp$.

		\FOR{$t$\textbf{ in} $0\dots T-1$, $i$\textbf{ in} $1\dots n$, \textbf{in parallel}}
		\STATE Sample $\xxi_i^{(t)}$, compute $\gg_i^{(t)} := \clip{\nabla \ell(\xx_i^{(t)}; C}, \xxi_i^{(t)})$, where $\clip{\gg; C} \coloneqq \min \left\{ 1, \, \tfrac{C}{\norm{\gg}} \right\} \cdot \gg$\;
            \STATE Sample for all $j \in \cN_i$, $\vv_{ij}^{(t)} = - \vv_{ji}^{(t)} \sim \cN(0, \sigmacor^2 \mI_d),$ and $ \overline{\vv}_i^{(t)}   \sim \cN(0, \sigmacdp^2 \mI_d)$
            \STATE $\Tilde{\gg}_i^{(t)} := \gg_i^{(t)} + \sum_{j \in \cN_i} \vv_{ij}^{(t)} + \overline{\vv}_i^{(t)}$ \hfill
            $\triangleright$ privacy noise
		\STATE $\xx_i^{(t + \frac{1}{2})} = \xx_i^{(t)} - \eta_t \Tilde{\gg}_i^{(t)}$  \hfill $\triangleright$ stochastic gradient updates
		\STATE $\xx_i^{(t + 1)} := \sum_{j=1}^n \mW_{ij} \xx_j^{(t + \frac{1}{2})}$  \hfill  $\triangleright$ gossip averaging
		\ENDFOR

	\end{algorithmic}
\end{algorithm}

\alg is an iterative decentralized algorithm proceeding in $T$ iterations, whereby at each iteration $t \in [T]$, each user $i \in \{1,\ldots,n\}$ first computes and clips a stochastic gradient at the current local model $\xx_i^{(t)}$ (line 2 of Algorithm~\ref{algo}):
\begin{align*}
    \gg_i^{(t)} := \clip{\nabla \ell(\xx_i^{(t)}, \xxi_i^{(t)});C},
\end{align*}
where $\xxi_i^{(t)}$ is a data point sampled at random from user $i$'s dataset $\cD_i$, and clipping with threshold $C$ corresponds to $\clip{\gg;C} \coloneqq \min{\left\{1, \frac{C}{\norm{\gg}}\right\}} \cdot \gg$ for any vector $\gg \in \R^d$.
The clipping operation ensures that the sensitivity of the gradient, to a change in data, is bounded as required by DP.
Then, on line 4 of Algorithm~\ref{algo}, each user obfuscates the clipped gradient by adding privacy noise:
\begin{align}
    \label{eq:noise}
    \Tilde{\gg}_i^{(t)} :=  \gg_i^{(t)} + \sum_{j \in \cN_i} \vv_{ij}^{(t)} + \overline{\vv}_i^{(t)},
\end{align}
where is $\overline{\vv}_i^{(t)} \sim \cN(0, \sigmacdp^2 \mI_d)$ is independent Gaussian noise, $\cN_i$ is the set of neighbors of $i$ on graph $\cG$, and $\{\vv_{ij}^{(t)}\}_{j \in \cN_i}$ are pairwise-cancelling correlated Gaussian noise terms; they satisfy $\vv_{ij}^{(t)} = -\vv_{ji}^{(t)} \sim \cN(0, \sigmacor^2 \mI_d)$.
Then, on line 5, each user makes a local update with the obfuscated stochastic gradient to obtain:
\begin{align*}
    \xx_i^{(t + \frac{1}{2})} = \xx_i^{(t)} - \eta_t \Tilde{\gg}_i^{(t)},
\end{align*}
where $\eta_t$ is the iteration's learning rate.
Finally, on line 6, each user broadcasts the obtained local model to its neighbors on graph $\cG$, and updates its local model by performing a weighted average of the neighbors' local models:
\begin{align}
    \label{eq:gossip}
    \xx_i^{(t + 1)} := \sum_{j=1}^n \mW_{ij} \xx_j^{(t + \frac{1}{2})},
\end{align}
where the weights are zero for non-neighboring users and form the mixing matrix $\mW = [\mW_{ij}]_{i,j\in [n]}\in \R^{n \times n}$, which is symmetric and doubly stochastic (see Definition~\ref{def:valid_mixing} below).
The motivation for injecting correlated noise in~\eqref{eq:noise} is that the gossip averaging in~\eqref{eq:gossip} will cancel out part or all correlated noise terms.
For example, if $\cG$ is the fully connected graph and $\mW = \tfrac{1}{n}\1 \1^\top$ is the matrix of ones times $\tfrac{1}{n}$, then \eqref{eq:noise} cancels out all correlated noise terms.
Still, the uncorrelated noise term $\bar \vv^{(t)}_i$ remains to protect the privacy of the gossip-averaged local model $\xx_i^{(t+1)}$.

%% file: theory.tex
\section{Privacy Analysis}
In this section, we formalize the privacy guarantees of \alg and introduce its privacy accountant.

First, we recall the notion of algebraic connectivity $a{(\cG)}$. Formally, algebraic connectivity $a{(\cG)}$ is equal to the second-smallest eigenvalue of the Laplacian matrix of the graph. 
Intuitively, it quantifies how well a graph is connected~\citep{fiedler1973algebraic}. %
For example, denoting by $n$ the number of vertices, the algebraic connectivity of the fully-connected graph is equal to $n$; for the star graph, it is equal to $1$; and for the ring graph it is equal to $2(1-\cos{\frac{2\pi}{n}}) = \Theta(\tfrac{1}{n^2})$ (see~\cite{de2007old} for a survey).

In this section, it is more convenient to work with SecRDP, the stronger variant of SecLDP based on Rényi DP (RDP)~\cite{mironov2017renyi}. We defer its formal definition to the appendix, and simply note that SecRDP implies SecLDP in the same way that RDP implies DP~\cite{mironov2017renyi}.

For brevity, in Theorem~\ref{th:privacy} we state the privacy guarantees of a single step of \alg against adversaries I and II (defined in Section~\ref{sec:problem}). We defer the extension of the privacy guarantees for adversaries III and IV to the appendix.

\begin{restatable}{theorem}{thprivacy}
\label{th:privacy}
Let $\alpha > 1$.
Each iteration of \alg (Algorithm~\ref{algo}) satisfies $(\alpha, \alpha \varepsilon)$-SecRDP (Definition~\ref{def:rdp}) against
\begin{itemize}
    \item an external eavesdropper with
    \begin{align*}
 \varepsilon \leq 2 C^2 \left( \frac{1}{n \sigmacdp^2} + \frac{1 - \tfrac{1}{n}}{\sigmacdp^2 + \ac{\cG} \sigmacor^2}  \right),
\end{align*}
\item honest-but-curious non-colluding users with
    \begin{align*}
 \varepsilon \leq 2 C^2 \left( \frac{1}{(n-1)\sigmacdp^2} + \frac{1 - \tfrac{1}{n-1}}{\sigmacdp^2 + a_{1}{(\cG)} \sigmacor^2}  \right),
\end{align*}
\end{itemize}
where $a_{1}{(\cG)}$ is the minimum algebraic connectivity across subgraphs obtained by deleting a single vertex from $\cG$.
Moreover, $\varepsilon$ can be computed numerically with Algorithm~\ref{algo:account}.
\end{restatable}

The result of Theorem~\ref{th:privacy} implies that in order to obtain SecRDP and thus SecLDP, i.e., in order to bound $\varepsilon$, it is sufficient for the noise magnitudes to scale as $\sigmacdp^2 = \Omega(\tfrac{1}{n})$, and $\sigmacor^2 = \Omega(\tfrac{1}{a(\cG)})$ or $\sigmacor^2 = \Omega(\tfrac{1}{a_1(\cG)})$ depending on the adversary strength.
Also, as we shall see in the next section, $\sigmacdp^2$ drives the dominant convergence terms, so its dependence on $n$ is crucial, while that of $\sigmacor^2$ only influences higher-order convergence terms.
If $\cG$ is disconnected, then its algebraic connectivity is zero~\cite{de2007old}, and one should set $\sigmacdp^2 = \Theta(1)$ and $\sigmacor = 0$ in \alg, which corresponds to the LDP baseline.

Finally, we compare our privacy analysis of a single step of \alg with its counterpart in the work of~\citet{sabater2022accurate}. Their Theorem~1 states a DP guarantee, while \alg guarantees RDP, which is stronger~\cite{mironov2017renyi}.  Moreover, our result applies generically to \emph{arbitrary} graph topologies, while theirs necessitates a graph-dependent analysis to derive the theoretical values of $\sigmacdp$ and $\sigmacor$.
For worst-case connected graphs, such as the ring graph, their analysis has a tighter dependence in terms of $\sigmacor$ after converting from RDP to DP, although the latter only marginally influences the privacy-utility trade-off.

\textbf{Privacy accountant.}
The theoretical privacy bound from Theorem~\ref{th:privacy} may be too loose for practical use.
Thus, we devise a privacy accounting method, described in Algorithm~\ref{algo:account}, which allows computing tight privacy bounds for a single step of \alg.
The accounting procedure is simple, and mainly involves computing the inverse of a ``modified'' graph Laplacian matrix, which can be conducted efficiently for large sparse graphs~\cite{vishnoi2012laplacian}.
It is straightforward to account the privacy for the full \alg procedure using the composition and DP conversion properties of RDP~\cite{mironov2017renyi} in addition to Algorithm~\ref{algo:account}. 

\begin{algorithm}[t]
    \caption{\textsc{Single-step SecRDP accountant}}\label{algo:account}
	\begin{algorithmic}[1]
		\REQUIRE clipping threshold $C$, noise variances $\sigmacdp, \sigmacor$.
            \IF{external eavesdropper}
            \STATE Get Laplacian matrix $\mL$ of the full graph $\cG$\;
            \STATE Compute $\SSigma = \left(\sigmacdp^2 \mI_{n} + \sigmacor^2 \mL \right)^{-1}$\;
            \STATE \textbf{return} $2C^2 \max_{i \in [n]} \SSigma_{ii}$\;
            \ENDIF
            \IF{honest-but-curious non-colluding users}
            \FOR{$i$ \textbf{in} $1 \ldots n$}
            \STATE Get Laplacian matrix $\mL$ of the subgraph of $\cG$ obtained by deleting vertex $i$\;
            \STATE Compute $\SSigma = \left(\sigmacdp^2 \mI_{n-1} + \sigmacor^2 \mL \right)^{-1}$\;
            \STATE $\varepsilon_i = 2C^2 \max_{j \in [n-1]} \SSigma_{jj}$\;
            \ENDFOR
            \STATE \textbf{return} $\max_{i \in [n]} \varepsilon_i$\;
            \ENDIF
	\end{algorithmic}
\end{algorithm}

\section{Utility Analysis}

In this section, we present our theoretical convergence and privacy-utility trade-off results. We first state our optimization assumptions below.

\subsection{Assumptions}
For all our theoretical results, we assume that the local loss functions are smooth.

\begin{assumption}[$L$-smoothness]\label{a:lsmooth}
Each function $\loss_i$
is differentiable and there exists a constant $L \geq 0$ such that for each $\xx, \yy \in \R^d, i \in [n]$:
\begin{align} \textstyle
&\norm{\nabla \loss_i(\yy) - \nabla \loss_i(\xx)}_2 \leq L \norm{\xx -\yy}_2\,. \label{eq:F-smooth}%
\end{align}
\end{assumption}

Additionally, some of our results require the Polyak-Łojasiewicz (PL) inequality~\citep{karimi2016linear}.
This condition does not require convexity, and is implied by strong convexity for example.
\begin{assumption}[$\mu$-PL]
\label{a:PL}
Function $\loss$ satisfies the $\mu$-Polyak-Łojasiewicz (PL) inequality. That is, for all $\xx \in \R^d$:
\begin{align} \textstyle
 2\mu(\loss(\xx) - \loss_\star) \leq \norm{\nabla \loss(\xx)}_2^2, \label{eq:PL}
\end{align}
where $\loss_\star \coloneqq \inf_{\xx \in \R^d} \loss(\xx)$ denotes the infimum of $\loss$.
\end{assumption}

\label{sec:noise}
We now formulate our conditions on the stochastic gradient noise and local loss functions heterogeneity.
\begin{assumption}[Bounded noise and heterogeneity]\label{a:opt_nc}
	We assume that there exist $P$, $\zeta_\star$ such that for all $\xx \in \R^d$,
	\begin{align} \textstyle
	\frac{1}{n} \sum_{i = 1}^n \norm{\nabla \loss_i(\xx)}_2^2 \leq \zeta_\star^2 + P \norm{\nabla \loss(\xx)}^2_2 \,, \label{eq:grad_opt_nc}
	\end{align}
    Also, we assume that there exist  $M$, $ \sigma_\star $ such that for all $\xx_1, \dots \xx_n \in \R^d$,
	\begin{align} \textstyle
	 \Psi(\xx_1,\ldots,\xx_n) \leq \sigma_\star^2 
 +  \frac{M}{n} \sum_{i=1}^n \norm{\nabla \loss(\xx_i)}^2_2,\label{eq:noise_opt_nc}
	\end{align}
	where we introduced $\Psi(\xx_1,\ldots,\xx_n) := \frac{1}{n} \sum_{i = 1}^n \EE{\xxi_i}{\norm{\nabla \ell(\xx_i, \xxi_i) - \nabla \loss_i(\xx_i)}}^2_2$.
\end{assumption}
Our noise assumption recovers the uniformly bounded noise assumption when $M=0$ and $n=1$, which is common for the non-convex analysis of SGD~\citep{bottou2018optimization}.
Our gradient heterogeneity assumption is one of the weakest in the literature~\citep{karimireddy2020scaffold}.
For the smooth convex (or PL) case, these assumptions hold with $\zeta_\star^2$ and $\sigma_\star^2$ being the gradient heterogeneity and noise, respectively, at the minimum only~\citep{vaswani2019fast}.

We additionally assume that gradients are bounded. This is a common assumption in private optimization to ignore the effect of clipping~\citep{agarwal2018cpsgd,noble2022differentially,allouah2023privacy}, which is not the focus of our work.
\begin{assumption}[Bounded Gradients]\label{a:boundedgradient}
We assume that there exists $C \geq 0$ such that for each $i \in [n], \xx \in \R^d, \xxi \in \cD_i$,
\begin{align}  
	\norm {\nabla \ell(\xx, \xxi)} \leq C.
	\end{align} 
\end{assumption}

As is typical in decentralized optimization algorithms, we make use of a mixing matrix $\mW$, as defined below.
\begin{definition}[Mixing matrix]\label{def:valid_mixing}
A matrix $\mW \in [0,1]^{n \times n}$ is a mixing matrix if it is symmetric and stochastic ($\mW\1 = \1$).
\end{definition}
Finally, we assume that the mixing matrix $\mW$ brings any set of vectors closer to their average with factor at least $1-p$.
\begin{assumption}[Consensus rate]\label{a:avg_distrib}
We assume that there exists $p \in (0,1]$ such that for every matrix $\mX \in \R^{d \times n}$,
\begin{align} \textstyle
	\norm{\mX \mW - \bar \mX}_F^2 &\leq (1 - p) \norm{\mX-\bar \mX}_F^2, \label{eq:p}
	\end{align}
	where we define the average
 $\bar{\mX} := \mX \tfrac{\1\1^\top}{n}$.
\end{assumption}
This assumption holds with $1-p$ being the second-largest eigenvalue value of $\mW\mW^\top$~\citep{boyd2006randomized}, e.g.,
$p = 1$ for the complete graph, $p=\Theta(\tfrac{1}{n^2})$ for the ring graph.

\subsection{Convergence Analysis}
We now present the convergence rate of~\alg, showing how the two the privacy noises affect its convergence speed.
First, we introduce the following quantity
\begin{equation}
    \label{def:hg}
    \hg{\mW} \coloneqq \frac{\sum_{i, k = 1}^{n} \norm{\mW_i - \mW_k}^2 \1_{k \in \cN_i}}{2\sum_{i, k = 1}^{n} \1_{k \in \cN_i}},
\end{equation}
where $\mW_i$ denotes the $i$-th column of $\mW$ and ``$k \in \cN_i$'' denotes $\{i,k\} \in \cE$.
This quantity naturally appears when analyzing the correlated noise reduction after one gossiping step.
It measures the heterogeneity of mixing weights $\mW$ across connected users of the graph $\cG$. The smaller $\hg{\mW}$, the closer are mixing weights, and the less correlated noise remains after one gossiping step, e.g., $\hg{\mW}=0$ for the complete graph with uniform mixing weights $\mW = \tfrac{\1\1^\top}{n}$.
More generally, for any graph $\cG$ with minimal degree $k_{\mathrm{min}} \geq 1$, we can show that $\hg{\mW}$ decreases with the minimal degree as $\hg{\mW} \leq \tfrac{2}{k_{\mathrm{min}}}$, when using uniform mixing weights, i.e., if $\mW_{ij}= \tfrac{\1_{j \in \cN_i}}{\mathrm{deg}(i)+1}, \forall i,j\in [n]$, where $\mathrm{deg}(i) = \card{\cN_i}$ is the degree of user $i$ in the graph.

We now state our convergence result in Theorem~\ref{thm:summary} below.

\begin{restatable}{theorem}{thsummary}\label{thm:summary}
Let Assumptions~\ref{a:lsmooth}, \ref{a:opt_nc}, \ref{a:boundedgradient}, \ref{a:avg_distrib} hold.
Consider Algorithm~\ref{algo}. 
Denote $\bar \xx^{(t)} = \frac{1}{n}\sum_{i = 1}^n \xx_i^{(t)}$,
$\loss_0 \coloneqq \loss(\bar{\xx}^{(0)}) - \loss_\star$, $\Xi_0 \coloneqq \tfrac{1}{n}\sum_{i=1}^n \|\xx_i^{(0)}-\bar{\xx}^{(0)}\|_2^2$, and $c \coloneqq \max{\{4\sqrt{3(1-p)(3P+pM)}, \tfrac{\mu}{L}, 2p, \tfrac{4pM}{n}}\}$.
For $T \geq 1$:
\begin{enumerate}[leftmargin=*]
    \item If $\loss$ is $\mu$-PL (Assumption~\ref{a:PL}) and $\eta_t = \tfrac{16}{\mu(t+c\tfrac{L}{\mu p})}$, then
\begin{align*}
    &\E\loss(\bar{\xx}^{(T)}) - \loss_\star
    \lesssim \frac{L(\sigma_\star^2+d\sigmacdp^2)}{\mu^2nT} + \frac{c^2L^2\loss_0}{\mu^2p^2T^2} \\
    &\quad+ \frac{cL^3 \Xi_0}{\mu^2 p^2 T^2}  + \frac{L^2\log{T}}{\mu^3pT^2}\Big((1-p)(\frac{\zeta_\star^2}{p} + \sigma_\star^2)\\
    &\quad+ \frac{\hg{\mW} \card{\cE} d \sigmacor^2}{n} + \norm{\mW-\tfrac{\1\1^\top}{n}}_F^2 d \sigmacdp^2\Big).
\end{align*}
    \item If $\eta_t = \min{\{\tfrac{p}{2c L}, 2\sqrt{\tfrac{\loss_0n}{LT(\sigma_\star^2+d\sigmacdp^2)}}\}}$
    , then
    \begin{align*}
    &\frac{1}{T} \sum_{t=0}^{T-1} \E\|\nabla \loss(\bar{\xx}^{(t)})\|_2^2
    \lesssim \sqrt{\frac{L\loss_0(\sigma_\star^2+d\sigmacdp^2)}{nT}}+ \frac{c L\loss_0}{p T} \\
    &\quad+ \frac{L^2\Xi_0}{pT} + \frac{L\loss_0n}{pT(\sigma_\star^2+d\sigmacdp^2)} \Big((1-p)(\frac{\zeta_\star^2}{p} +\sigma_\star^2)\\
    &\quad+ \frac{\hg{\mW} \card{\cE} d \sigmacor^2}{n} + \norm{\mW-\tfrac{\1\1^\top}{n}}_F^2 d \sigmacdp^2\Big).
\end{align*}
\end{enumerate}
In the above, $\lesssim$ denotes inequality up to absolute constants.
\end{restatable}

In Theorem~\ref{thm:summary}, the leading (slowest) terms of the convergence rates are in $\mathcal{O}{(\tfrac{1}{T})}$ and $\mathcal{O}{(\tfrac{1}{\sqrt{T}})}$ in the PL and non-convex cases, respectively. Thus, our analysis recovers the optimal asymptotic convergences rates in stochastic optimization~\cite{agarwal2009information,arjevani2023lower}.
Moreover, it features a linear speedup in the number of users $n$, like vanilla decentralized SGD~\cite{koloskova2020unified}.

The main difference with decentralized SGD is in the non-dominant terms due to the injected correlated noise. In the PL case for example, the term that depends on the correlated noise scales as
\begin{equation}
\label{eq:slowdown1}
\widetilde{\mathcal{O}}{\left(\frac{\hg{\mW} \card{\cE} d \sigmacor^2}{n T^2}\right)},
\end{equation}
by ignoring privacy-independent constants and logarithmic terms.
The above term quantifies a \emph{slowdown} effect of correlated noise.
Interestingly, it is non-dominant in $T$, and proportional to $\hg{\mW}$.
For example, this term is zero for the complete graph with $\mW = \tfrac{\1\1^\top}{n}$, because $\hg{\mW}=0$ in this case, which is expected as all correlated noise terms should be cancelled after one gossiping step.

\subsection{Privacy-utility Trade-off}
We now combine our privacy and convergence analyses to quantify the privacy-utility trade-off of \alg.
We recall that a graph is $2$-connected if it remains connected after removing any vertex.
We present the privacy-utility trade-off of \alg
in Corollary~\ref{cor:final} below by focusing on the PL case. We defer the non-convex result to the appendix because lower bounds are unknown in this case.

\begin{restatable}{corollary}{cortradeoff}
\label{cor:final}

Let Assumptions~\ref{a:lsmooth}-\ref{a:avg_distrib} hold.
Let $\varepsilon>0, \delta \in (0,1)$ be such that $\varepsilon \leq \log{(1/\delta)}$.
Algorithm~\ref{algo} satisfies $(\varepsilon, \delta)$-SecLDP (Definition~\ref{def:ldp}) with expected error 
\begin{align*}
        \mathcal{O}\left(\frac{C^2 d \log{(1/\delta)}}{n^2 \varepsilon^2}\right),
    \end{align*}
against the following adversaries:
\begin{itemize}
    \item an external eavesdropper: if $\cG$ is connected, $\sigmacdp^2= \tfrac{32 C^2 T \log{(1/\delta)}}{n \varepsilon^2}$ and $\sigmacor^2= \tfrac{32 C^2 T\log{(1/\delta)}}{a(\cG)\varepsilon^2}$,
    \item honest-but-curious non-colluding users: if $\cG$ is $2$-connected, $\sigmacdp^2= \tfrac{32 C^2 T \log{(1/\delta)}}{(n-1) \varepsilon^2}$ and $\sigmacor^2= \tfrac{32 C^2 T\log{(1/\delta)}}{a_1(\cG)\varepsilon^2}$, where $a_{1}{(\cG)}$ is the minimum algebraic connectivity across subgraphs obtained by deleting a single vertex from $\cG$.
\end{itemize}
In the above, $\mathcal{O}$ omits absolute constants, vanishing terms in $T$, and privacy-independent multiplicative constants $L, \mu$.
\end{restatable}

\textbf{Tightness.}
The lower bound on the privacy-utility trade-off under user-level CDP is $\Omega{\left(\tfrac{d}{n^2\varepsilon^2}\right)}$~\citep{bassily2014private}.\footnote{We refer to the strongly convex lower bound of \citet{bassily2014private}, which also applies to the (larger) PL class of functions.} 
Under LDP, the lower bound on the privacy-utility trade-off is $\Omega{\left(\tfrac{d}{n\varepsilon^2}\right)}$~\citep{duchi2018minimax}.
Therefore, following the result of Corollary~\ref{cor:final}, \alg matches the optimal CDP privacy-utility trade-off, under SecLDP against both an external eavesdropper and non-colluding curious users.
We recall that this improves by factor $n$ over the trade-off achieved by LDP algorithms~\cite{bellet2018personalized,cheng2019towards,huang2019dp,li2023convergence}.
Besides, for comparison, \citet{cyffers2022muffliato} derive a privacy-utility trade-off in $\mathcal{O}{\left(\tfrac{k_{\mathrm{max}}}{\sqrt{p}n^2 \varepsilon^2}\right)}$, where $k_{\mathrm{max}}$ is the maximum degree of the graph, for a relaxation of Network DP~\cite{cyffers2022privacy}. 
Their trade-off matches CDP for well-connected graphs such as expanders~\cite{ying2021exponential}, but degrades with poorer connectivity, e.g., $\mathcal{O}{\left(\tfrac{1}{n \varepsilon^2}\right)}$ for the ring graph. In contrast, our trade-off matches CDP for arbitrary connected graphs, albeit the privacy definitions are orthogonal in general, as we discuss in Section~\ref{sec:problem}.
We also extend Corollary~\ref{cor:final} to colluding curious users (adversary III in Section~\ref{sec:problem}) in the appendix and match the optimal privacy-utility trade-off when there is a constant fraction of colluding users.
Naturally, if the group of colluding users is too large, the threat model of SecLDP approaches that of LDP, and thus cannot match the privacy-utility trade-off of CDP in such cases. 

Although the trade-off achieved by \alg is tight asymptotically in the number of iterations $T$, we recall that the convergence speed suffers from a slowdown due to correlated noise following our discussion of the term~\eqref{eq:slowdown1} of Theorem~\ref{thm:summary}.
Moreover, our analysis shows a tight trade-off by making stronger assumptions on the connectivity of the graph for stronger adversaries.
That is, for non-colluding curious users, Corollary~\ref{cor:final} assumes the graph to be $2$-connected, i.e., it remains connected after removing any vertex, while it only assumes the graph to be connected for the external eavesdropper. We believe this condition to be necessary: in the worst-case where a curious user $i$ is the unique neighbor of another user $j$, then $i$ can substract the correlated noise injected by user $j$, and leave the latter with the uncorrelated noise only, which is insufficient to protect its local model.

%% file: experiments.tex
\section{Empirical Evaluation}
\label{sec:exp}

\begin{figure*}[t]
\vspace{-0.3cm}
\centering
\begin{subfigure}[t]{0.9\textwidth}
\centering
\includegraphics[width=\textwidth]{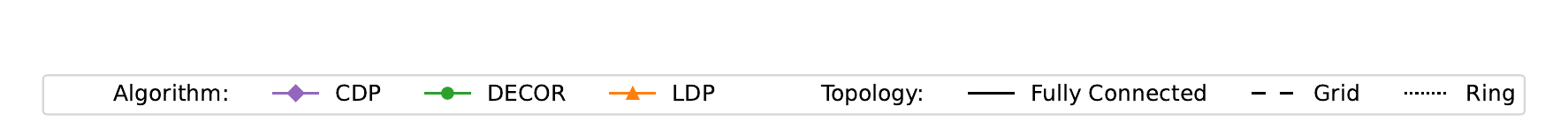}
\end{subfigure}
\vspace{-0.3cm}
\begin{subfigure}[t]{0.33\textwidth}
\centering
\includegraphics[width=\textwidth]{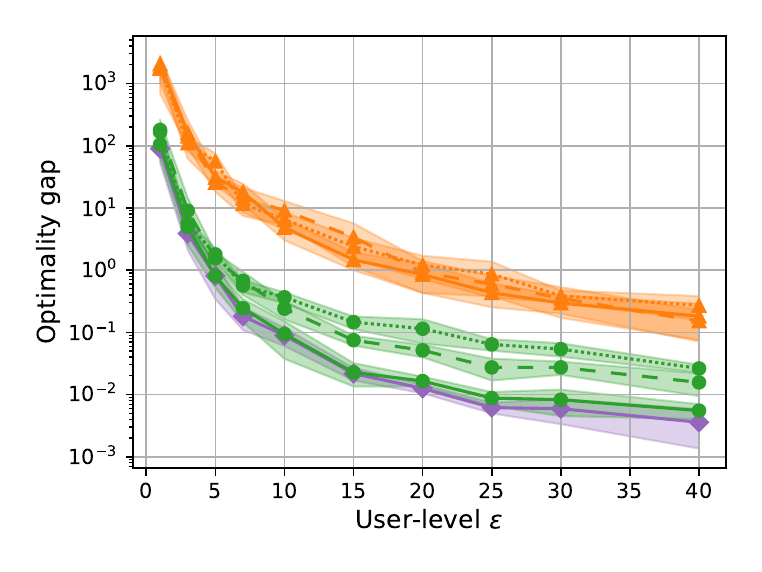}
\vspace{-0.8cm}
\caption{Least-squares regression on synthetic data} %
\label{fig:quadratics}
\end{subfigure}
\begin{subfigure}[t]{0.33\textwidth}
\centering
\includegraphics[width=\textwidth]{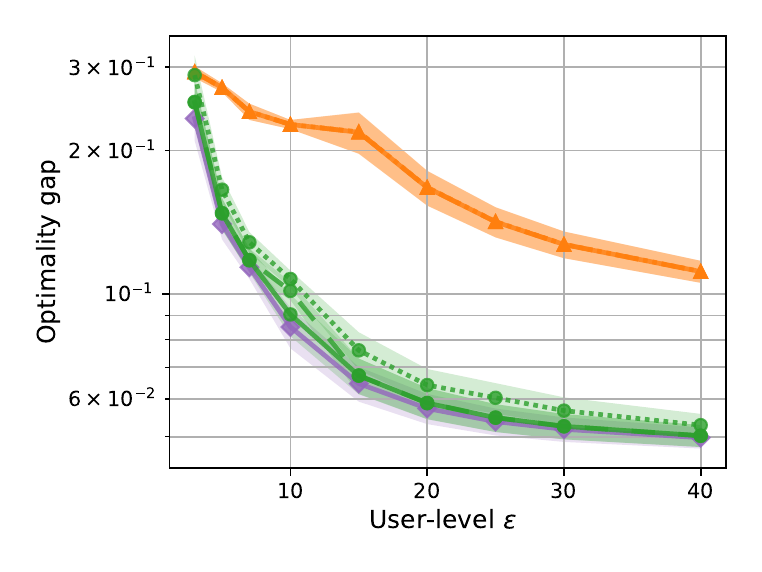}
\vspace{-0.8cm}
\caption{Logistic regression on a9a}
\label{fig:libsvm}
\end{subfigure}
\begin{subfigure}[t]{0.33\textwidth}
\centering
\includegraphics[width=\textwidth]{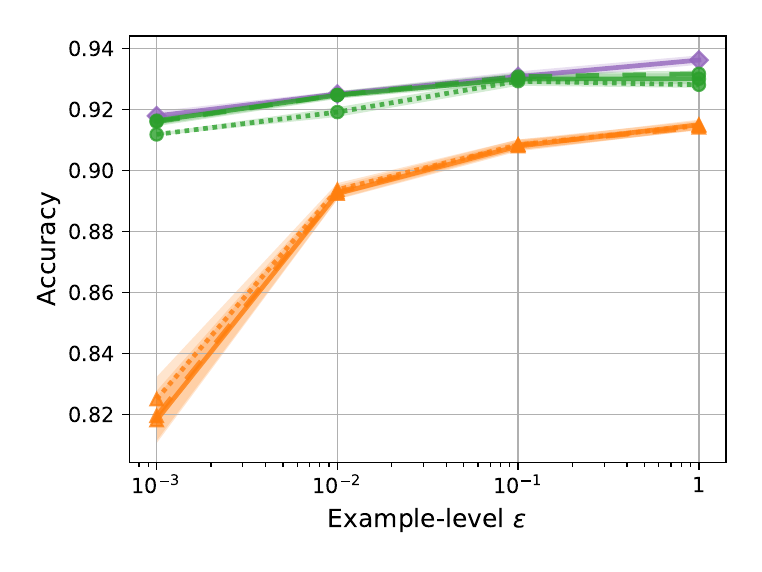}
\vspace{-0.8cm}
\caption{Neural network training on MNIST}
\label{fig:mnist}
\end{subfigure}
\vspace{-0.25cm}
\caption{Privacy-utility trade-offs for \alg and the CDP and LDP baselines on least-squares regression, logistic regression, and neural network training under $(\varepsilon, 10^{-5})$-SecLDP against an external eavesdropper observing all communications. \alg closely matches the performance of CDP, and considerably surpasses LDP, across all considered tasks, privacy budgets, and topologies.}
\label{fig:acc_privacy}
  \vspace{-0.4cm}
\end{figure*}

In this section, we empirically show that \alg achieves a privacy-utility trade-off matching the CDP baseline, and surpassing the LDP baseline.
Recall that LDP is the strongest threat model in decentralized learning, while CDP is the weakest, and thus they represent lower and upper bounds in terms of performance.
We compare these algorithms on three strongly convex and non-convex tasks with synthetic and real-world datasets, across various user-level privacy budgets and network topologies.
For simplicity, we focus on adversary I of SecLDP, i.e., an external eavesdropper observing all user communications.
We provide a summary of our empirical evaluations in Figure~\ref{fig:acc_privacy}.

\textbf{Setup.}
We consider $n=16$ users on three usual network topologies in increasing connectivity: ring, grid (2d torus), and fully-connected. We use the Metropolis-Hastings~\cite{boyd2006randomized} mixing matrix, i.e., $\mW_{ij}= \tfrac{\1_{j \in \cN_i}}{\mathrm{deg}(i)+1}, \forall i,j\in [n]$, where $\mathrm{deg}(i) = \card{\cN_i}$ is the degree of user $i$ in the graph. We tune all hyperparameters for each algorithm individually, and run each experiment with four seeds for reproducibility.
We account for the privacy budget using our SecLDP privacy accountant (Algorithm~\ref{algo:account}).
We defer the full experimental setup to the appendix.
\footnote{Our code is available at \href{https://github.com/elfirdoussilab1/DECOR}{https://github.com/elfirdoussilab1/DECOR}.}

\subsection{Strongly Convex Tasks}
We study two strongly convex tasks: least-squares regression on synthetic data and regularized logistic regression on the \textit{a9a} LIBSVM dataset~\cite{chang2011libsvm}. Both of these tasks are covered by our theory in the PL case.

\textbf{Least-squares regression.}
In this task, the empirical loss function is given for every $i\in [n]$ as follows:
$$ \loss_i(\xx) = \frac{1}{2} \| \mA_i\xx - \bb_i \|_2^2,~\forall \xx \in \R^d,$$
where $\mA_i^2 \coloneqq \tfrac{i^2}{n} \mI_d $ and $\bb_i \sim \cN(0, \tfrac{1}{i^2} \mI_d)$.
As shown in Figure~\ref{fig:quadratics}, \alg achieves an order of magnitude lower training loss than LDP across all privacy levels and topologies. Moreover, the performance of \alg is comparable to that of CDP, especially towards the low privacy regime. 

\textbf{Logistic regression.}
In this task, the loss function at data point $(\aa,b) \in \R^d \times \{0,1\}$ is given as:
\begin{align*}
    \ell(\xx, (\aa, b)) =  \log(1+b\exp(- \xx^\top \aa )) + \lambda \| \xx\|_2^2, \forall \xx \in \R^d.
\end{align*}
In Figure~\ref{fig:libsvm}, similar to the previous task, \alg closely matches the performance of CDP across all considered topologies and privacy budgets, while being an order of magnitude better than LDP.

\subsection{Non-convex Task}
We consider the MNIST~\cite{mnist} task with a one-hidden-layer neural network. Since this task is more challenging for the LDP baseline, we consider example-level DP, i.e., privacy is ensured when changing any single data point, to which our theoretical results can be extended straightforwardly using privacy amplification by subsampling~\cite{wang2019subsampled}.

In Figure~\ref{fig:mnist}, we once again observe that \alg matches CDP, and surpasses LDP, on all considered topologies and privacy budgets. Indeed, the gap of CDP with LDP is almost $10$ accuracy points for the lowest privacy budget, as suggested by the theory, while the gap between \alg on the ring topology and the CDP baseline, or \alg on the grid topology, is less than $1$ accuracy point.

%% file: conclusion.tex
\section{Conclusion}
Decentralized learning is a promising paradigm for scalable data usage without compromising privacy. However, the risk of privacy breaches by curious users remains. Our solution \alg addresses this challenge by securely exchanging randomness seeds and injecting pairwise correlated Gaussian noises to protect local models. Our theoretical and empirical results show that \alg matches the optimal centralized privacy-utility trade-offs, while ensuring differential privacy within the strong threat model of SecLDP, our new relaxation of LDP. Our proposed SecLDP offers protection against external eavesdroppers and curious users, assuming shared secrets among connected pairs. 
We propose a privacy accountant with our algorithm in order to foster reproducibility and encourage further research and deployment of private decentralized learning algorithms.

\section*{Acknowledgments}
This work was supported in part by SNSF grant 200021\_200477.
The authors are thankful to the anonymous reviewers for their constructive comments.

\section*{Impact Statement}
Our work is mainly theoretical, we do not think any specific societal consequences must be highlighted here.

%% file: appendix.tex
\section*{APPENDIX}
The appendix is organized as follows. 
The proofs and extensions of our privacy analysis, including Theorem~\ref{th:privacy}, are in Appendix~\ref{sec:app-privacy}.
The proofs and extensions of our convergence analysis, including Theorem~\ref{thm:summary}, are in Appendix~\ref{sec:app-convergence}.
The proofs and extensions of our privacy-utility trade-off, including Corollary~\ref{cor:final}, are in Appendix~\ref{sec:app-tradeoff}.
Our detailed experimental setup is in Appendix~\ref{sec:app-expsetup}.

\input{appendix_theory}
\input{appendix_experiments}

%% file: appendix_theory.tex
\section{Privacy Analysis}
\label{sec:app-privacy}

In this section, we prove our main privacy result stated in Theorem~\ref{th:privacy} and extend it to the general privacy adversaries discussed in Section~\ref{sec:problem}.
We first recall some useful facts around Rényi divergences and linear algebra. 

\begin{definition}
[$\alpha$-Rényi divergence]
Let $\alpha>0, \alpha \neq 1$. The $\alpha$-Rényi divergence between two probability distributions $P$ and $Q$ is defined as
\begin{align*}
    \mathrm{D}_{\alpha}{\left ( P ~\|~ Q \right )} \coloneqq \frac{1}{\alpha-1}\log{\EE{X \sim Q}{\left(\frac{P{(X)}}{Q{(X)}}\right)^\alpha}}.
\end{align*}
\end{definition}

\begin{lemma}[\protect{\cite{gil2013renyi}}]
\label{lem:renyigaussian}
Let $\alpha>0 ,\alpha \neq 1, \mu_1, \mu_2 \in \R^n$, and $\SSigma \in \R^{n \times n}$. Assume that $\SSigma$ is positive definite.
The $\alpha$-Rényi divergence between the multivariate Gaussian distributions $\cN(\mu_1, \SSigma)$ and $\cN(\mu_2, \SSigma)$ is
\begin{align*}
    \mathrm{D}_{\alpha}{\left(\cN(\mu_1, \SSigma)~\|~\cN(\mu_2, \SSigma)\right)} = \frac{\alpha}{2} (\mu_1 - \mu_2)^\top \SSigma^{-1} (\mu_1 - \mu_2).
\end{align*}
\end{lemma}

We recall the folklore result below, which is a consequence of the Courant-Fischer min-max theorem~\citep{de2007old}.
\begin{lemma}
\label{lem:courantfischer}
Let $\mM \in \R^{n \times n}$ be a real symmetric matrix and $\uu_n \in \R^n$ be an eigenvector associated to the largest eigenvalue of $\mM$.
The second-largest eigenvalue of $\mM$ is
\begin{align*}
    \lambda_{n-1}(\mM) = \sup_{\substack{\uu \neq \0 \\ \lin{\uu,\uu_n}=0}} \frac{\uu^\top \mM \uu}{\norm{\uu}_2^2}.
\end{align*}
\end{lemma}

We now define \emph{secret-based local Rényi differential privacy} (SecRDP), a strong variant of SecLDP based on Rényi DP.
\begin{definition}[SecRDP]
    \label{def:rdp}
    Let $\varepsilon \geq 0$, $\delta \in [0, 1]$, $\alpha>1$. Consider a randomized decentralized algorithm $\cA : \cX^{m \times n} \to \cY$, which outputs the transcript of all communications. Algorithm $\cA$ is said to satisfy $(\alpha,\varepsilon, \mathcal{S})$-SecRDP if $\cA$ satisfies $(\alpha, \varepsilon)$-RDP given that $\mathcal{S}$ is unknown to the adversary. That is, for every adjacent datasets $\cD, \cD' \in \cX^{m \times n}$, we have
    \begin{equation*}
        \mathrm{D}_{\alpha}{\left(\cA{(\cD)} ~\middle|~ \cS~\text{is hidden} ~\middle\|~\cA{(\cD')} ~\middle|~ \cS~\text{is hidden}\right)} \leq \varepsilon,
    \end{equation*}
    where the left-hand side is the Rényi divergence (Definition~\ref{def:rdp}) between the probability distributions of $\cA{(\cD)}$ and $\cA{(\cD')}$, conditional on the secrets $\cS$ being hidden from the adversary.
    We simply say that $\cA$ satisfies $(\alpha,\varepsilon)$-SecRDP if it satisfies $(\alpha,\varepsilon, \mathcal{S})$-SecRDP for a certain $\mathcal{S}$.
\end{definition}
Both SecLDP and SecRDP preserve the properties of DP and RDP, respectively, since these relaxations only condition the probability space of the considered distributions.

\textbf{Proof of Theorem~\ref{th:privacy}.}
For convenience, we restate Theorem~\ref{th:privacy} below, whose proof is a special case of the extended privacy result given next.

\thprivacy*
\begin{proof}
    The result is a special case of Theorem~\ref{th:privacy-general}, by taking $q=0$ for the external eavesdropper and $q=1$ for the honest-but-curious non-colluding users.
\end{proof}

\textbf{Extended privacy analysis.}
We now state and prove a general privacy analysis of \alg to all considered adversaries in Section~\ref{sec:problem}, which includes collusion. We additionally provide a SecRDP accountant in Algorithm~\ref{algo:account-general}, which generalizes Algorithm~\ref{algo:account} to the aforementioned adversaries.
\begin{algorithm}[t]
    \caption{\textsc{General SecRDP accountant for \alg}}\label{algo:account-general}
	\begin{algorithmic}[1]
		\REQUIRE clipping threshold $C$, noise variances $\sigmacdp, \sigmacor$, collusion level $q$.

            \FOR{$\cI \subseteq [n], \card{\cI} = q$}
            \STATE Get Laplacian matrix $\mL$ of the subgraph of $\cG$ after deleting vertices $\cI$\;
            \STATE Compute $\SSigma = \left(\sigmacdp^2 \mI_{n-q} + \sigmacor^2 \mL \right)^{-1}$\;
            \STATE $\varepsilon_\cI = 2C^2 \max_{i \in [n-q]} \SSigma_{ii}$\;
            \ENDFOR
            \STATE \textbf{return} $\max_{\cI \subseteq [n], \card{\cI} = q} \varepsilon_\cI$\;
	\end{algorithmic}
\end{algorithm}

\begin{theorem}
\label{th:privacy-general}
Let $\alpha > 1$ and $q<n$.
Each iteration of Algorithm~\ref{algo} satisfies $(\alpha, \alpha \varepsilon)$-SecRDP (Definition~\ref{def:rdp}) against honest-but-curious users colluding at level $q$ with
\begin{align}
\label{eq:privacy}
 \varepsilon \leq 2 C^2 \left( \frac{1}{(n-q)\sigmacdp^2} + \frac{1 - \tfrac{1}{n-q}}{\sigmacdp^2 + a_{q}{(\cG)} \sigmacor^2}  \right),
\end{align}
where $a_{q}{(\cG)}$ is the minimum algebraic connectivity across subgraphs obtained by deleting $q$ vertices from $\cG$.
Moreover, $\varepsilon$ can be computed numerically using Algorithm~\ref{algo:account-general}.
\end{theorem}
\begin{proof}
Let $\alpha>1$, $q>1$, and $\cI \subseteq [n]$ be an arbitrary group of $\card{\cI} = q$ users.
Recall that we denote by $\cS_\cI \coloneqq \left\{s_{jk} \colon \{j,k\} \in \cE, j,k \notin \cI \right\}$ the set of secrets hidden from all users in $\cI$.
We will prove that Algorithm~\ref{algo} satisfies $(\alpha, \varepsilon, \cS_{\cI})$-SecRDP, which protects against honest-but-curious users colluding at level $q$, as discussed in Section~\ref{sec:problem}.
For ease of exposition, we consider the one-dimensional case $d=1$. Extending the proof to the general case is straightforward.

Formally, at each iteration of Algorithm~\ref{algo}, users possess private inputs (gradients) in the form of vector $\xx \in [-C,C]^n$, given that gradients are clipped at threshold $C$. Each user $i \in [n]$ shares the following privatized quantity:
\begin{align}
    \label{eq:privatized1}
    \widetilde{\xx}_i \coloneqq \xx_i + \sum_{j \in \cN_i} \vv_{ij} + \bar{\vv}_{i},
\end{align}
where $\vv_{ij} = -\vv_{ji} \sim \cN(0, \sigmacor^2)$ for all $j \in \cN_i$, and $\bar{\vv}_{i} \sim \cN(0, \sigmacdp^2)$. Note that each neighborhood $\cN_i$ does not include $i$.

Denote by $\cH \coloneqq [n] \setminus \cI$ the set of the $\card{\cH}=n-q$ honest (non-colluding) users.
Our goal is to show that the mechanism producing $\widetilde{\mX}_{\cH} \coloneqq \big[\widetilde{\xx}_i\big]_{i \in \cH}$ satisfies SecRDP when a single entry of $\mX \coloneqq \big[\xx_i\big]_{i \in \cH}$ is arbitrarily changed; i.e., one user's input differs.
To do so, we first rewrite~\eqref{eq:privatized1} to discard the noise terms known to the colluding curious users who can simply substract them to get for every $i \in \cH$: 
\begin{align}
    \label{eq:privatized2}
    \widetilde{\xx}_i = \xx_i + \sum_{j \in \cN_i \cap \cH} \vv_{ij} + \bar{\vv}_{i},
\end{align}
Denote by $\cG_\cH \coloneqq (\cH, \cE_\cH)$ the subgraph of $\cG$ restricted to honest users.
We now rewrite the above in matrix form as:
\begin{align}
\label{eq:vectorform}
    \widetilde{\mX}_{\cH} = \mX_{\cH} + \mK \mN_{\cE} + \bar{\mN},
\end{align}
where $\mK \in \R^{(n-q) \times \card{\cE_{\cH}}}$ is the oriented incidence matrix of the graph $\cG_{\cH}$ and $\mN_{\cE_{\cH}} = [\vv_{ij}]_{1 \leq i < j \leq n-q} \in \R^{\card{\cE_{\cH}}}$ is the vector of pairwise noises.
Now, consider two input vectors $\mX_A, \mX_B \in [-C,C]^{n-q}$ which differ maximally in an arbitrary coordinate $i \in [n-q]$ without loss of generality: 
\begin{equation}
    \label{eq:xa-xb}
    \mX_A - \mX_B = 2C \ee_i \in \R^{n-q},
\end{equation}
where $\ee_i$ is the vector of $\R^{n-q}$ where the only nonzero element is $1$ in the $i$-th coordinate.

We will then show that the $\alpha$-Rényi divergence between $\widetilde{\mX}_A$ and $\widetilde{\mX}_B$, which are respectively produced by input vectors $\mX_A$ and $\mX_B$, is bounded.
To do so, by looking at Equation~\eqref{eq:vectorform}, we can see that $\widetilde{\mX}_A, \widetilde{\mX}_B$ follow a multivariate Gaussian distribution of means $\mX_A, \mX_B$ respectively and of variance
\begin{align}
    \label{eq:Sigma}
    \SSigma \coloneqq \E (\widetilde{\mX}_A - \mX_A) (\widetilde{\mX}_{A} - \mX_A)^\top 
    = \E (\widetilde{\mX}_B -  \mX_B)(\widetilde{\mX}_{B} -  \mX_B)^\top
    = \sigmacor^2 \mL + \sigmacdp^2 \mI_{n-q}  \in \R^{(n-q) \times (n-q)},
\end{align}
where $\mL = \mK \mK^\top \in \R^{(n-q) \times (n-q)}$ is the Laplacian matrix of the graph $\cG_\cH$~\citep{de2007old}.
Note that $\SSigma$ is positive definite when $\sigmacdp^2>0$ because $\mL$ is positive semi-definite.

Therefore, following Lemma~\ref{lem:renyigaussian}, the $\alpha$-Rényi divergence between the distributions of $\widetilde{\mX}_A$ and $\widetilde{\mX}_B$ is
\begin{align}
\label{eq:renyi1}
D_\alpha(\widetilde{\mX}_A~\|~\widetilde{\mX}_B) = \frac{\alpha}{2} (\mX_A - \mX_B)^\top \SSigma^{-1} (\mX_A - \mX_B).
\end{align}
Now, recall that the spectrum of $\mL$ is $0=\lambda_1(\mL) \leq \ldots \leq \lambda_{n-q}(\mL)$ because it is the Laplacian matrix of the graph $\cG_\cH$. Moreover, the eigenvector corresponding to the zero eigenvalue is $\1 \in \R^{n-q}$ the vector of ones.
Thus, since $\SSigma$ is a real symmetric (positive definite) matrix, the spectrum of $\SSigma^{-1}$ in ascending order is $\left(\tfrac{1}{\sigmacdp^2 + \sigmacor^2 \lambda_{n-q-i+1}(\mL)}\right)_{i \in [n-q]}$, and $\1$ the vector of ones is associated to its largest eigenvalue:
\begin{align}
\label{eq:sigmaeigvector}
    \SSigma^{-1} \1 = \frac{1}{\sigmacdp^2} \1.
\end{align}
Define $\xx_i \coloneqq \ee_i - \frac{1}{n-q}\1$ and observe that $\lin{\xx_i,\1}=0$.
Therefore, we can decompose the vector $\mX_A - \mX_B$ from Equation~\eqref{eq:xa-xb} as a sum of orthogonal vectors as follows:
\begin{align*}
    \mX_A - \mX_B = 2C \ee_i = 2C \left(\frac{1}{n-q} \1 + \xx_i \right).
\end{align*}
Going back to~\eqref{eq:renyi1}, we can write
\begin{align}
D_\alpha(\widetilde{\mX}_A~\|~\widetilde{\mX}_B) 
&= \frac{\alpha}{2} (\mX_A - \mX_B)^\top \SSigma^{-1} (\mX_A - \mX_B) 
= \frac{4 C^2 \alpha}{2} (\frac{1}{n-q} \1 + \xx_i)^\top \SSigma^{-1} (\frac{1}{n-q} \1 + \xx_i)\nonumber \\
&= 2\alpha C^2 \left( \frac{1}{(n-q)^2} \1^\top \SSigma^{-1} \1 + \xx_i^\top \SSigma^{-1} \xx_i  \right)
= 2\alpha C^2 \left( \frac{1}{(n-q) \sigmacdp^2} + \xx_i^\top \SSigma^{-1} \xx_i  \right),
\label{eq:renyi2}
\end{align}
where we have used that $\lin{\xx_i,\1}=0$, Equation~\eqref{eq:sigmaeigvector} and $\lin{\1,\1}=n-q$
successively in the last two steps.
Now, using Lemma~\ref{lem:courantfischer} and the facts that $\lin{\xx_i,\1}=0$ and $\norm{\xx_i}_2^2 = 1 - \frac{1}{n-q}$, we have that
\begin{align*}
    \xx_i^\top \SSigma^{-1} \xx_i \leq \sup_{\substack{\uu \neq \0 \\ \lin{\uu,\1}=0}} \frac{\uu^\top \SSigma^{-1} \uu}{\norm{\uu}_2^2}  \cdot \norm{\xx_i}_2^2 
    \leq \lambda_{n-q-1}(\SSigma^{-1}) \norm{\xx_i}_2^2
    = (1 - \frac{1}{n-q}) \lambda_{n-q-1}(\SSigma^{-1})
    = \frac{1 - \frac{1}{n-q}}{\sigmacdp^2 + \lambda_2(\mL) \sigmacor^2}.
\end{align*}
Plugging the bound above back in~\eqref{eq:renyi2}, we obtain
\begin{align*}
D_\alpha(\widetilde{\mX}_A~\|~\widetilde{\mX}_B)
\leq 2\alpha C^2 \left( \frac{1}{(n-q) \sigmacdp^2} + \frac{1 - \frac{1}{n-q}}{\sigmacdp^2 + \lambda_2(\mL) \sigmacor^2}  \right).
\end{align*}
Recall that $\lambda_2(\mL)$ is the algebraic connectivity of the graph $\cG_\cH$, by definition.
Moreover, since $\cI$ and thus $\cH$ are taken arbitrarily, in the worst case $\lambda_2(\mL)$ is $a_{q}(\cG)$ the minimum algebraic connectivity across subgraphs obtained by deleting $q$ vertices from $\cG$. This
concludes the proof of~\eqref{eq:privacy} the main result.

Finally, it is easy to see from Equation~\eqref{eq:renyi1} that the exact privacy bound $\varepsilon$ can be computed numerically using Algorithm~\ref{algo:account}.
Indeed, the maximal difference in inputs is $\mX_A - \mX_B = 2C \ee_i$ for some $i \in [n-q]$ as in~\eqref{eq:xa-xb}, so the maximal privacy bound, given that $\cI$ is the set of colluding users, is
\begin{align*}
    \varepsilon_\cI = \max_{i \in [n-q]}~ \frac{1}{2}(2C\ee_i)^\top \SSigma^{-1} (2C\ee_i) = 2C^2 \max_{i \in [n-q]}~  \ee_i^\top \SSigma^{-1} \ee_i = 2C^2 \max_{i \in [n-q]}~ \SSigma^{-1}_{ii},
\end{align*}
where $\SSigma^{-1}_{ii}$ is the $i$-th entry in the diagonal of the inverse of $\SSigma = \sigmacor^2 \mL + \sigmacdp^2 \mI_{n-q}$. Thus, to get the maximal privacy loss across all possible colluding user groups of size $q$, we take $\varepsilon = \max_{\cI \subseteq [n], \card{\cI} = q} \varepsilon_\cI$. Observing that the latter is exactly the output of Algorithm~\ref{algo:account-general} concludes the proof.
\end{proof}

\section{Convergence Analysis}
\label{sec:app-convergence}

In this section, we prove our convergence analysis stated in Theorem~\ref{thm:summary} and extend it to the general privacy adversaries discussed in Section~\ref{sec:problem}.
We introduce some useful notation in Section~\ref{sec:conv-notation}, overview the main elements of the proof in Section~\ref{sec:conv-overview}, prove the main theorem in Section~\ref{sec:conv-proof}, and finally prove the intermediate lemmas in Section~\ref{sec:conv-lemmas}. 

\subsection{Notation}
\label{sec:conv-notation}
We can rewrite the procedure of \alg (Algorithm~\ref{algo}) using the following matrix notation, extending the definition used in Section~\ref{sec:algorithm}:
\begin{align}
\begin{split}
&\mX^{(t)} := \left[ \xx_1^{(t)},\dots, \xx_n^{(t)}\right] \in \R^{d\times n}, \qquad
\bar{\mX}^{(t)} := \left[ \bar{\xx}^{(t)},\dots, \bar{\xx}^{(t)}\right]  \in \R^{d\times n}, \\  
&\partial \ell(\mX^{(t)}, \xxi^{(t)}) := \left[\nabla \ell(\xx_{1}^{(t)}, \xxi_1^{(t)}), \dots,  \nabla \ell(\xx_{n}^{(t)}, \xxi_n^{(t)})\right]  \in \R^{d\times n}, \\
&\mN^{(t)} := \left[ \sum_{j \in \cN_1} \vv_{1j}^{(t)},\dots, \sum_{j \in \cN_n} \vv_{nj}^{(t)}\right]  \in \R^{d\times n}, \qquad
\bar{\mN}^{(t)} := \left[ \bar{\vv}_1^{(t)},\dots, \bar{\vv}_n^{(t)}\right]  \in \R^{d\times n}.
\label{eq:matrix_notation}
\end{split}
\end{align}
We recall that under the bounded gradient assumption (Assumption~\ref{a:boundedgradient}), clipping leaves gradients unaffected, and thus we discard the clipping operator in this section.

\begin{algorithm}[H]
    \caption{\textsc{\alg in matrix notation}}\label{alg:matrix_notation}
	\begin{algorithmic}[1]
		\REQUIRE for each user $i\in [n]$ initialize $\xx_i^{(0)} \in \R^d$, 
		 stepsizes $\{\eta_t\}_{t=0}^{T-1}$, number of iterations $T$, 
		 mixing matrix $\mW$,
        noise parameters $\sigmacor$ and $\sigmacdp$.

		\FOR{$t$\textbf{ in} $0\dots T-1$}
  \STATE $\mX^{(t + \frac{1}{2})} = \mX^{(t)} - \eta_t\left(\partial \ell(\mX^{(t)}, \xxi_i^{(t)}) + \mN^{(t)} + \bar{\mN}^{(t)} \right)$ \hfill $\triangleright$ stochastic gradient updates
		\STATE $\mX^{(t + 1)} = \mX^{(t + \frac{1}{2})} \cdot \mW $ \hfill  $\triangleright$ gossip averaging
		\ENDFOR
	\end{algorithmic}
\end{algorithm}

\subsection{Proof Overview}
\label{sec:conv-overview}

Our convergence analysis relies upon three elements: descent bound, pairwise noise reduction, and consensus distance recursion.
We first state the corresponding lemmas, and defer their proofs to Section~\ref{sec:conv-lemmas}.

The first proof element is the descent bound of Lemma~\ref{lem:decrease_nc}. It quantifies the progress made after each \alg step. In particular, compared to the error due to stochastic gradient variance $\sigma_\star^2$ as in vanilla SGD~\cite{bottou2018optimization}, there are two additional quantities involved:  and (uncorrelated) privacy noise variance $\sigmacdp^2$, and the consensus distance $\Xi_t$ defined for every $t \geq 1$ as
\begin{equation}
    \Xi_{t} \coloneqq \frac{1}{n}{\sum_{i=1}^n  \E\norm{\xx_i^{(t)}-\bar \xx^{(t)}}^2} = \frac{1}{n} \norm{\mX^{(t)}-\bar{\mX}^{(t)}}_F^2.
\end{equation}
\begin{restatable}[Descent bound]{lemma}{descent}\label{lem:decrease_nc} Under Assumptions~\ref{a:lsmooth}, \ref{a:opt_nc} and \ref{a:avg_distrib},
	the averages $\bar{\xx}^{(t)} := \frac{1}{n}\sum_{i=1}^n \xx_i^{(t)}$ of the iterates of Algorithm~\ref{algo} with  $\eta_t \leq \frac{1}{2L}\min{\{1, \frac{n}{2M}\}}$ satisfy 
	\begin{align}
        \E{\left[\loss(\bar{\xx}^{(t + 1)}) - \loss(\bar{\xx}^{(t)})\right]} 
    & \leq -\frac{\eta_t}{4} \E\norm{\nabla \loss(\bar{\xx}^{(t)})}_2^2  + \frac{3\eta_t L^2}{4} \Xi_t
    +\frac{L\eta_t^2}{2} \frac{\sigma_\star^2 + d \sigmacdp^2}{n}.
	\end{align}
\end{restatable}
Interestingly, the descent bound does not involve the correlated noise variance $\sigmacor^2$.
This is thanks to the correlated noise terms cancelling out pairwise, so that the correlated noise disappears when analyzing the average model $\overline{\xx}^{(t)}$.

Next, in order to bound the consensus distance $\Xi_t$, we first quantify in Lemma~\ref{lem:pairwisenoisereduce} the effect of correlated noise in a single step of \alg on the consensus distance $\Xi_t$.
\begin{restatable}[Correlated noise reduction]{lemma}{noisereduce}
\label{lem:pairwisenoisereduce}
Consider Algorithm~\ref{algo}.
For any undirected graph $\cG = (\{1, \ldots, n\}, \cE)$ and any matrix $\mW \in \R^{n \times n}$ and at every iteration $t$, we have
\begin{align}
    \EE{}{\norm{\mN^{(t)} \mW}_F^2} 
    = \hg{\mW} \cdot \EE{}{\norm{\mN^{(t)}}_F^2}
    = 2 \hg{\mW} \card{\cE} d \sigmacor^2,
\end{align}
where we define 
$\hg{\mW} \coloneqq \tfrac{\sum_{i, k = 1}^{n} \norm{\mW_i - \mW_k}^2 \1_{k \in \cN_i}}{2\sum_{i, k = 1}^{n} \1_{k \in \cN_i}},$
and $\1_{k \in \cN_i}$ denotes $\{i,k\} \in \cE$, and $\card{\cE} = \frac{1}{2}\sum_{i, k = 1}^{n} \1_{k \in \cN_i}$ is the number of edges on the graph $\cG$. Moreover, if $\mW_{ij}= \tfrac{\1_{j \in \cN_i}}{\mathrm{deg}(i)+1}, \forall i,j\in [n]$, where $\mathrm{deg}(i) = \card{\cN_i}$ is the degree of user $i$ in the graph, we have $\hg{\mW} \leq \tfrac{2}{k_{\mathrm{min}}}$,
where $k_{\mathrm{min}} \geq 1$ is the minimal degree of graph $\cG$.
\end{restatable}
The analysis of the error due to correlated noise in Lemma~\ref{lem:pairwisenoisereduce} is exact, in the sense that it is an equality. Recall that $\hg{\mW}$ is graph- and mixing matrix-dependent. Broadly speaking, the average expected error per edge (due to correlated noise) is $d \sigmacor^2$ (the variance of one correlated noise term), reduced by factor $\hg{\mW}$, which decreases with the connectivity with the graph.
Using this lemma, we can now prove a powerful recursion on the consensus distance in Lemma~\ref{lem:consensus} below.

\begin{restatable}[Consensus distance recursion]{lemma}{consensus}\label{lem:consensus}
	Under Assumptions~\ref{a:lsmooth}, \ref{a:opt_nc}, and \ref{a:avg_distrib}, if in addition stepsizes satisfy $\eta_t \leq  \frac{p}{L\sqrt{6(1-p)(3+pM)}}$, then 
	\begin{align*}
    \Xi_{t+1}
    &\leq (1-\frac{p}{2}) \Xi_t
    +  2\eta_t^2(1-p)(\frac{3P}{p}+M) \E\norm{\nabla \loss(\bar{\xx}^{(t)})}_2^2\\
    &\quad + \eta_t^2 \left[6(1-p)\frac{\zeta_\star^2}{p} + (1-p)\sigma_\star^2 + \frac{2\hg{\mW} |\cE| d \sigmacor^2}{n}
    + \norm{\mW - \frac{\1\1^\top}{n}}_F^2 d \sigmacdp^2\right],
	\end{align*}
	where $\Xi_{t} \coloneqq \frac{1}{n}{\sum_{i=1}^n  \E\norm{\xx_i^{(t)}-\bar \xx^{(t)}}^2}$ is the consensus distance.
\end{restatable}
The effects of the privacy noises are apparent in the lemma above, and correspond mainly to the quantity analyzed in Lemma~\ref{lem:pairwisenoisereduce}, in addition to the effects of stochastic variance and heterogeneity, which are similar to vanilla D-SGD~\cite{koloskova2020unified}. It is indeed intuitive that the non-cancelled correlated noise should pull the local models away, and this worsens for poorly-connected graphs.

\subsection{Main Proof}
\label{sec:conv-proof}
We now restate and prove Theorem~\ref{thm:summary} below, using the intermediate lemmas from the previous section.
\thsummary*

\subsubsection{PL case}
\begin{proof}
Let assumptions~\ref{a:lsmooth}-\ref{a:avg_distrib} hold.
Consider Algorithm~\ref{algo} with the stepsize sequence defined for every $t \geq 0$ as:
\begin{equation}
    \eta_t \coloneqq \frac{16}{\mu(t+c \tfrac{L}{\mu p})},
    \label{eq:schedule-PL}
\end{equation}
where $c \coloneqq \max{\{4\sqrt{3(1-p)(3P+pM)}, \tfrac{\mu}{L}, 2p, \tfrac{4pM}{n}}\}$. Clearly, this sequence is decreasing and we have for every $t\geq0$: $$\eta_t \leq \eta_0 = \min{\{\frac{p}{4L\sqrt{3(1-p)(3P+pM)}}, \frac{p}{\mu}, \frac{1}{2L}\min{\{1, \frac{n}{2M}\}}}\}.$$
This ensures that the conditions of lemmas~\ref{lem:decrease_nc} and~\ref{lem:consensus} are verified.

Consider the sequence defined for every $t \geq 0$ as:
\begin{equation}
    V_t \coloneqq \E{\left[\loss(\bar{\xx}^{(t)}) - \loss_\star \right]} + \frac{3L^2 \eta_t}{p} \Xi_t,
    \label{eq:lyapunov1-PL}
\end{equation}
where $\loss_\star \coloneqq \inf_{\xx \in \R^d} \loss(\xx)$ denotes the infimum of $\loss$.
Clearly, since $\Xi_t$ is also non-negative as a sum of squared distances, we have $V_t \geq 0$ for every $t \geq 0$.
We also define the following auxiliary sequence for every $t \geq 0$:
\begin{equation}
    W_t \coloneqq \frac{1}{\eta_t^2} V_t.
    \label{eq:lyapunov2-PL}
\end{equation}
Fix $t \geq 0$.
First, to analyze $W_t$, we write 
\begin{align*}
    W_{t+1} - W_t
    &= \frac{1}{\eta_{t+1}^2} V_{t+1} - \frac{1}{\eta_{t}^2} V_{t}
    = \frac{1}{\eta_{t+1}^2} (V_{t+1} - \frac{\eta_{t+1}^2}{\eta_t^2} V_t).
\end{align*}
Moreover, denoting $\hat{t} \coloneqq t+\tau$, we have $\tfrac{\eta_{t+1}^2}{\eta_t^2} = \tfrac{\hat{t}^2}{(\hat{t}+1)^2} = 1 - \tfrac{1+2\hat{t}}{(\hat{t}+1)^2}$.
Thus, we have
\begin{align}
    W_{t+1} - W_t
    = \frac{1}{\eta_{t+1}^2} (V_{t+1} - (1 - \frac{1+2\hat{t}}{(\hat{t}+1)^2})V_t).
    \label{eq:PL1}
\end{align}
On the other hand, to analyze $V_t$, we use the fact that stepsizes are non-increasing and satisfy the conditions of lemmas~\ref{lem:decrease_nc} and~\ref{lem:consensus}:
\begin{align*}
    V_{t+1} - V_t
    &= \E{\left[\loss(\bar{\xx}^{(t+1)}) - \loss(\bar{\xx}^{(t)}) \right]} + \frac{3L^2}{p} (\eta_{t+1}\Xi_{t+1}-\eta_{t}\Xi_{t})\\
    &\leq \E{\left[\loss(\bar{\xx}^{(t+1)}) - \loss(\bar{\xx}^{(t)}) \right]} + \frac{3L^2 \eta_{t}}{p} (\Xi_{t+1}-\Xi_{t})\\
    &\leq \left(-\frac{\eta_t}{4} + \frac{6L^2 \eta_t^3}{p}(1-p)(\frac{3P}{p} + M)\right)\E\norm{\nabla \loss(\bar{\xx}^{(t)})}_2^2
    + (\frac{3\eta_t L^2}{4} - \frac{3L^2\eta_t}{2})\Xi_t
    + \eta_t^2 A + \eta_t^3 B,
\end{align*}
where we introduced $A \coloneqq \tfrac{L}{2}\tfrac{\sigma_\star^2+d\sigmacdp^2}{n}$ and $B \coloneqq \tfrac{3L^2}{p}\left(6(1-p)\tfrac{\zeta_\star^2}{p} + (1-p)\sigma_\star^2 + \tfrac{2\hg{\mW} \card{\cE} d \sigmacor^2}{n} + \norm{\mW-\tfrac{\1\1^\top}{n}}_F^2 d \sigmacdp^2\right)$ for simplicity.
Recall that we have $\eta_t \leq \eta_0 \leq \tfrac{p}{4L\sqrt{3(1-p)(3P+pM)}}$, so that $\tfrac{6L^2}{p} \eta_t^2 (1-p)(\tfrac{3P}{p} + M) \leq \tfrac{1}{8}$.
Consequently, we have
\begin{align}
    V_{t+1} - V_t
    &\leq -\frac{\eta_t}{8} \E\norm{\nabla \loss(\bar{\xx}^{(t)})}_2^2
    - \frac{3L^2\eta_t}{4}\Xi_t
    + \eta_t^2 A + \eta_t^3 B.
    \label{eq:PL2}
\end{align}
Now, recall that $\norm{\nabla \loss(\bar{\xx}^{(t)})}_2^2 \geq 2\mu (\loss(\bar{\xx}^{(t)}) - \loss_\star)$ following Assumption~\ref{a:PL}, so that the bound above becomes
\begin{align*}
    V_{t+1} - V_t
    &\leq -\frac{\mu \eta_t}{4}(\E\loss(\bar{\xx}^{(t)}) - \loss_\star)
    - \frac{3L^2\eta_t}{4}\Xi_t
    + \eta_t^2 A + \eta_t^3 B\\
    &= -\frac{\mu \eta_t}{4} \left(\E\loss(\bar{\xx}^{(t)}) - \loss_\star
    + \frac{3L^2}{\mu}\Xi_t\right)
    + \eta_t^2 A + \eta_t^3 B.
\end{align*}
Recall also that $\eta_t 
\leq \eta_0 \leq \tfrac{p}{\mu}$. Therefore, we have
\begin{align*}
    V_{t+1} - V_t
    &\leq -\frac{\mu \eta_t}{4} \left(\E\loss(\bar{\xx}^{(t)}) - \loss_\star
    + \frac{3L^2 \eta_t}{p}\Xi_t\right)
    + \eta_t^2 A + \eta_t^3 B
    = -\frac{\mu \eta_t}{4} V_t
    + \eta_t^2 A + \eta_t^3 B.
\end{align*}
Plugging the above bound back in~\eqref{eq:PL1} and then substituting $\eta_t = \tfrac{16}{\mu\hat{t}}$, we get
\begin{align*}
    W_{t+1} - W_t
    &= \frac{1}{\eta_{t+1}^2} (V_{t+1} - (1 - \frac{1+2\hat{t}}{\hat{t}^2})V_t)
    \leq \frac{1}{\eta_{t+1}^2} \left(-\frac{\mu \eta_t}{4} V_t
    + \eta_t^2 A + \eta_t^3 B + \frac{1+2\hat{t}}{\hat{t}^2}V_t\right)\\
    &= -\mu^2(\hat{t}+1)^2 \left(\frac{4}{ \hat{t}} - \frac{1+2\hat{t}}{\hat{t}^2}\right) V_t + \frac{(\hat{t}+1)^2}{\hat{t}^2} A + \frac{16(\hat{t}+1)^2}{\mu\hat{t}^3} B.
\end{align*}
Observe that $\hat{t} = t+c \tfrac{L}{\mu p}\geq c \tfrac{L}{\mu p}$, so that $\frac{4}{ \hat{t}} - \frac{1+2\hat{t}}{\hat{t}^2} = \frac{2}{ \hat{t}} - \frac{1}{\hat{t}^2} \leq \frac{1}{ \hat{t}}$ and $\tfrac{(\hat{t}+1)^2}{\hat{t}^2} = 1 + \frac{1+2\hat{t}}{\hat{t}^2} \leq 4$. Therefore, the bound above becomes
\begin{align*}
    W_{t+1} - W_t
    &\leq -\mu^2\frac{(\hat{t}+1)^2}{\hat{t}} V_t + 4A + \frac{64}{\mu \hat{t}} B \leq 4A + \frac{64}{\mu \hat{t}} B.
\end{align*}
By summing over $t \in \{0, \ldots, T-1\}$ and substituting $A$ and $B$, we get
\begin{align*}
    W_{T} - W_0
    &\leq 2LT\frac{\sigma_\star^2+d\sigmacdp^2}{n}+\left(\sum_{t=0}^{T-1} \hat{t}\right)\frac{192L^2}{\mu p}\left(6(1-p)\frac{\zeta_\star^2}{p} + (1-p)\sigma_\star^2 + \frac{2\hg{\mW} \card{\cE} d \sigmacor^2}{n} + \norm{\mW-\frac{\1\1^\top}{n}}_F^2 d \sigmacdp^2\right).
\end{align*}
We now substitute $W_T$ and $W_0$ to obtain $W_{T} - W_0 = \tfrac{1}{\eta_T^2}V_T - \tfrac{1}{\eta_0^2}V_0 = \tfrac{\mu^2}{256}((T+c \tfrac{L}{\mu p})^2V_T - \left(\tfrac{cL}{\mu p}\right)^2 V_0)$.
Also, as $c \tfrac{L}{\mu p} \geq \tfrac{2L}{\mu} \geq 2$, we have $\sum_{t=0}^{T-1} \tfrac{1}{\hat{t}} = \sum_{t=0}^{T-1} \tfrac{1}{t+c \tfrac{L}{\mu p}} \leq \ln(T+1)$.
Thus, after rearranging terms, the inequality above becomes
\begin{align*}
    (T+c \tfrac{L}{\mu p})^2V_T - \left(\tfrac{cL}{\mu p}\right)^2 V_0
    &\leq \frac{512LT}{\mu^2}\frac{\sigma_\star^2+d\sigmacdp^2}{n}\\
    &\quad+\ln(T+1)\frac{49152L^2}{\mu^3p}\left(6(1-p)\frac{\zeta_\star^2}{p} + (1-p)\sigma_\star^2 + \frac{2\hg{\mW} \card{\cE} d \sigmacor^2}{n} + \norm{\mW-\frac{\1\1^\top}{n}}_F^2 d \sigmacdp^2\right).
\end{align*}
Upon dividing both sides by $(T+c \tfrac{L}{\mu p})^2$, rearranging terms and recalling that $c \tfrac{L}{\mu p} \geq 1$, we obtain
\begin{align*}
    V_T &\leq \frac{\left(\tfrac{cL}{\mu p}\right)^2}{(T+c \tfrac{L}{\mu p})^2}V_0 + \frac{512LT}{(T+c \tfrac{L}{\mu p})^2\mu^2}\frac{\sigma_\star^2+d\sigmacdp^2}{n}\\
    &\quad + \frac{\ln(T+1)}{(T+c \tfrac{L}{\mu p})^2}\frac{49152L^2}{\mu^3p}\left(6(1-p)\frac{\zeta_\star^2}{p} + (1-p)\sigma_\star^2 + \frac{2\hg{\mW} \card{\cE} d \sigmacor^2}{n} + \norm{\mW-\frac{\1\1^\top}{n}}_F^2 d \sigmacdp^2\right)\\
    &\leq \frac{c^2L^2}{\mu^2p^2T^2}V_0 + \frac{512L}{\mu^2 T}\frac{\sigma_\star^2+d\sigmacdp^2}{n}\\
    &\quad + \frac{\ln(T+1)}{T^2}\frac{49152L^2}{\mu^3p}\left(6(1-p)\frac{\zeta_\star^2}{p} + (1-p)\sigma_\star^2 + \frac{2\hg{\mW} \card{\cE} d \sigmacor^2}{n} + \norm{\mW-\frac{\1\1^\top}{n}}_F^2 d \sigmacdp^2\right).
\end{align*}
Finally, we obtain the final result by substituting $V_T$, $V_0$ and $\eta_T, \eta_0$ and rearranging terms:
\begin{align*}
    &\E\loss(\bar{\xx}^{(T)}) - \loss_\star + \frac{48L^2}{\mu p(T+c \tfrac{L}{\mu p})} \Xi_T
    \leq  \frac{512L}{\mu^2 T}\frac{\sigma_\star^2+d\sigmacdp^2}{n}\\
    &\quad + \frac{\ln(T+1)}{T^2}\frac{49152L^2}{\mu^3p}\left((1-p)(6\frac{\zeta_\star^2}{p} + \sigma_\star^2) + \frac{2\hg{\mW} \card{\cE} d \sigmacor^2}{n} + \norm{\mW-\frac{\1\1^\top}{n}}_F^2 d \sigmacdp^2\right)\\
    &\quad+ \frac{c^2L^2(\loss(\bar{\xx}^{(0)}) - \loss_\star)}{\mu^2p^2T^2} + \frac{48cL^3}{\mu^2 p^2 T^2} \Xi_0.
\end{align*}
\end{proof}

\subsubsection{Non-convex case}
\begin{proof}
Let assumptions~\ref{a:lsmooth},~\ref{a:opt_nc},~\ref{a:boundedgradient}, and~\ref{a:avg_distrib} hold.
Consider Algorithm~\ref{algo} with the constant stepsize sequence defined for every $t \geq 0$ as:
\begin{align}
    \eta_t = \eta \coloneqq \min{\{\frac{p}{2 c L}, 2\sqrt{\frac{(\loss(\bar{\xx}^{(0)}) - \loss_\star)n}{LT(\sigma_\star^2+d\sigmacdp^2)}}\}},
    \label{eq:schedule-nc}
\end{align}
where $c \coloneqq \max{\{4\sqrt{3(1-p)(3P+pM)}, \tfrac{\mu}{L}, 2p, \tfrac{4pM}{n}}\}$.
This ensures that the conditions of lemmas~\ref{lem:decrease_nc} and~\ref{lem:consensus} are verified.

Consider the sequence defined for every $t \geq 0$ as:
\begin{equation}
    V_t \coloneqq \E{\left[\loss(\bar{\xx}^{(t)}) - \loss_\star \right]} + \frac{3L^2 \eta}{p} \Xi_t,
    \label{eq:lyapunov1-PL}
\end{equation}
where $\loss_\star \coloneqq \inf_{\xx \in \R^d} \loss(\xx)$ denotes the infimum of $\loss$.
Clearly, since $\Xi_t$ is also non-negative as a sum of squared distances, we have $V_t \geq 0$ for every $t \geq 0$.

Denote $A \coloneqq \tfrac{L}{2}\tfrac{\sigma_\star^2+d\sigmacdp^2}{n}$ and $B \coloneqq \tfrac{3L^2}{p}\left(6(1-p)\tfrac{\zeta_\star^2}{p} + (1-p)\sigma_\star^2 + \tfrac{2\hg{\mW} \card{\cE} d \sigmacor^2}{n} + \norm{\mW-\tfrac{\1\1^\top}{n}}_F^2 d \sigmacdp^2\right)$.
Following the same steps of the PL case until~\eqref{eq:PL2}, we have
\begin{align*}
    V_{t+1} - V_t
    &\leq -\frac{\eta}{8} \E\norm{\nabla \loss(\bar{\xx}^{(t)})}_2^2
    - \frac{3L^2\eta}{4}\Xi_t
    + \eta^2 A + \eta^3 B
    \leq -\frac{\eta}{8} \E\norm{\nabla \loss(\bar{\xx}^{(t)})}_2^2
    + \eta^2 A + \eta^3 B.
\end{align*}
By averaging over $t \in \{0,\ldots,T-1\}$, multiplying by $\tfrac{8}{\eta}$ and rearranging terms we obtain
\begin{align*}
    \frac{1}{T} \sum_{t=0}^{T-1} \E\norm{\nabla \loss(\bar{\xx}^{(t)})}_2^2
    &\leq \frac{8(V_0-V_T)}{\eta T}
    + 8\eta A + 8\eta^2 B.
\end{align*}
By recalling that $V_T \geq 0$ and substituting the values of $V_0$ and $A$, we get
\begin{align*}
    \frac{1}{T} \sum_{t=0}^{T-1} \E\norm{\nabla \loss(\bar{\xx}^{(t)})}_2^2
    &\leq \frac{8(\loss(\bar{\xx}^{(0)}) - \loss_\star + \frac{3L^2 \eta}{p} \Xi_0)}{\eta T}
    + 4\eta L\frac{\sigma_\star^2+d\sigmacdp^2}{n} + 8\eta^2 B\\
    &= \frac{8(\loss(\bar{\xx}^{(0)}) - \loss_\star)}{\eta T}
    + 4\eta L\frac{\sigma_\star^2+d\sigmacdp^2}{n} + 8\eta^2 B + \frac{24L^2}{pT} \Xi_0.
\end{align*}
Now, recalling the value of $\eta$, and that $\tfrac{1}{\eta} = \max{\{\tfrac{2cL}{p}, \tfrac{1}{2}\sqrt{\frac{TL(\sigma_\star^2+d\sigmacdp^2)}{(\loss(\bar{\xx}^{(0)}) - \loss_\star)n}}\}} \leq \tfrac{2cL}{p} + \tfrac{1}{2}\sqrt{\frac{TL(\sigma_\star^2+d\sigmacdp^2)}{(\loss(\bar{\xx}^{(0)}) - \loss_\star)n}}$.
Therefore, the bound above becomes
\begin{align*}
    \frac{1}{T} \sum_{t=0}^{T-1} \E\norm{\nabla \loss(\bar{\xx}^{(t)})}_2^2
    &\leq  \frac{16cL(\loss(\bar{\xx}^{(0)}) - \loss_\star)}{p T} + 4\sqrt{\frac{L(\loss(\bar{\xx}^{(0)}) - \loss_\star)(\sigma_\star^2+d\sigmacdp^2)}{nT}}\\
    &\quad + 8\sqrt{\frac{L(\loss(\bar{\xx}^{(0)}) - \loss_\star)(\sigma_\star^2+d\sigmacdp^2)}{nT}} + \frac{32(\loss(\bar{\xx}^{(0)}) - \loss_\star)n}{LT(\sigma_\star^2+d\sigmacdp^2)} B + \frac{24L^2}{pT} \Xi_0.
\end{align*}
By rearranging terms and substituting $B$, we obtain
\begin{align*}
    \frac{1}{T} \sum_{t=0}^{T-1} \E\norm{\nabla \loss(\bar{\xx}^{(t)})}_2^2
    &\leq 12\sqrt{\frac{L(\loss(\bar{\xx}^{(0)}) - \loss_\star)(\sigma_\star^2+d\sigmacdp^2)}{nT}} + \frac{16cL(\loss(\bar{\xx}^{(0)}) - \loss_\star)}{pT} + \frac{24L^2}{pT} \Xi_0\\
    &\quad + \frac{96L(\loss(\bar{\xx}^{(0)}) - \loss_\star)n}{pT(\sigma_\star^2+d\sigmacdp^2)} \left(6(1-p)\frac{\zeta_\star^2}{p} + (1-p)\sigma_\star^2 + \frac{2\hg{\mW} \card{\cE} d \sigmacor^2}{n} + \norm{\mW-\tfrac{\1\1^\top}{n}}_F^2 d \sigmacdp^2\right).
\end{align*}
The above concludes the proof.
\end{proof}

\subsection{Proof of Lemmas}
\label{sec:conv-lemmas}
We now restate and prove the intermediate lemmas from the previous sections.

\descent*
\begin{proof}
Let assumptions~\ref{a:lsmooth}, \ref{a:opt_nc}, and \ref{a:avg_distrib} hold.
Because mixing matrices preserve the average, as a direct consequence of Definition~\ref{def:valid_mixing}, we have
\begin{align*}
	\bar{\xx}^{(t + 1)} 
	&= \bar{\xx}^{(t)} - \frac{\eta_t}{n}\sum_{i = 1}^n  \Tilde{\gg}_i^{(t)}
  = \bar{\xx}^{(t)} - \frac{\eta_t}{n}\sum_{i = 1}^n  \left(\nabla \ell(\xx_i^{(t)}, \xxi_i^{(t)})
 + \sum_{j \in \cN_i} \vv_{i,j}^{(t)} +\overline{\vv}_i^{(t)}\right)\\
 &= \bar{\xx}^{(t)} - \frac{\eta_t}{n}\sum_{i = 1}^n  \nabla \ell(\xx_i^{(t)}, \xxi_i^{(t)})
 - \frac{\eta_t}{n}\sum_{i = 1}^n \sum_{j \in \cN_i} \vv_{i,j}^{(t)} - \frac{\eta_t}{n}\sum_{i = 1}^n \overline{\vv}_i^{(t)}.
	\end{align*}
Recall that for all $i \in [n], j \in \cN_i,$ we have $\vv_{i,j}^{(t)} = -\vv_{j, i}^{(t)}$, so that $\sum_{i = 1}^n \sum_{j \in \cN_i} \vv_{i,j}^{(t)} = 0$.
Reporting this in the equation above yields:
\begin{align}
    \label{eq:descent1}
	\bar{\xx}^{(t + 1)} 
	= \bar{\xx}^{(t)} - \frac{\eta_t}{n}\sum_{i = 1}^n  \nabla \ell(\xx_i^{(t)}, \xxi_i^{(t)}) - \frac{\eta_t}{n}\sum_{i = 1}^n \overline{\vv}_i^{(t)}.
	\end{align}
 Also, since function $\loss$ is $L$-smooth as the average of smooth functions (Assumption~\ref{a:lsmooth}), by taking conditional expectation $\E_t$ on all randomness up to iteration $t$, we have (see~\citep{bottou2018optimization}) 
	\begin{align}
	\EE{t}{\loss(\bar{\xx}^{(t + 1)})} 
	& \leq \loss(\bar{\xx}^{(t)}) +\underbrace{\EE{t}{\lin{ \nabla \loss(\bar{\xx}^{(t)}), \bar{\xx}^{(t + 1)} - \bar{\xx}^{(t)} }} }_{\eqqcolon A}+ \frac{L}{2} \eta_t^2 \underbrace{\EE{t}{ \norm{\bar{\xx}^{(t + 1)} - \bar{\xx}^{(t)}}_2^2}}_{\eqqcolon B}.
 \label{eq:descent2}
	\end{align}
We start by bounding $A$, by using~\eqref{eq:descent1} and the smoothness of $\loss_i$, as follows:
	\begin{align*}
	A
    &= -\eta_t \lin{ \nabla \loss(\bar{\xx}^{(t)}), \EE{}{\left[\frac{1}{n}\sum_{i = 1}^n  \nabla \ell(\xx_i^{(t)}, \xxi_i^{(t)}) + \frac{1}{n}\sum_{i = 1}^n \overline{\vv}_i^{(t)}\right]}}
    = -\eta_t \lin{ \nabla \loss(\bar{\xx}^{(t)}), \frac{1}{n}\sum_{i = 1}^n  \nabla \loss_i(\xx_i^{(t)})} \\
    &= \frac{\eta_t}{2} \left[ \norm{\frac{1}{n}\sum_{i = 1}^n  \nabla \loss_i(\xx_i^{(t)}) - \nabla \loss(\bar{\xx}^{(t)})}_2^2 - \norm{\nabla \loss(\bar{\xx}^{(t)})}_2^2 - \norm{\frac{1}{n}\sum_{i = 1}^n  \nabla \loss_i(\xx_i^{(t)})}_2^2 \right] \\ 
    &\leq \frac{\eta_t}{2} \left[ \frac{1}{n}\sum_{i = 1}^n\norm{\nabla \loss_i(\xx_i^{(t)}) - \nabla \loss_i(\bar{\xx}^{(t)})}_2^2 - \norm{\nabla \loss(\bar{\xx}^{(t)})}_2^2 - \norm{\frac{1}{n}\sum_{i = 1}^n  \nabla \loss_i(\xx_i^{(t)})}_2^2 \right] \\
    &\leq -\frac{\eta_t}{2} \norm{\nabla \loss(\bar{\xx}^{(t)})}_2^2 
    - \frac{\eta_t}{2}\norm{\frac{1}{n}\sum_{i = 1}^n  \nabla \loss_i(\xx_i^{(t)})}_2^2
    + \frac{\eta_t L^2}{2} \frac{1}{n}\sum_{i=1}^n \norm{\xx_i^{(t)} - \bar{\xx}^{(t)}}_2^2.
	\end{align*}
	For the last term $B$, using~\eqref{eq:descent1} and Assumption~\ref{a:opt_nc}, we obtain
	\begin{align*}
	B &= \E_{t}\norm{\frac{1}{n}\sum_{i = 1}^n  \nabla \ell(\xx_i^{(t)}, \xxi_i^{(t)}) + \frac{1}{n}\sum_{i = 1}^n \overline{\vv}_i^{(t)}}_2^2
        = \E\norm{\frac{1}{n}\sum_{i = 1}^n  \nabla \ell(\xx_i^{(t)}, \xxi_i^{(t)})}_2^2 + \E\norm{\frac{1}{n}\sum_{i = 1}^n \overline{\vv}_i^{(t)}}_2^2\\
 & =  \E\norm{\frac{1}{n} \sum_{j = 1}^n \left(\nabla \ell(\xx_i^{(t)}, \xi_i^{(t)})-\nabla \loss_i(\xx_i^{(t)}) \right)}_2^2 + \norm{\frac{1}{n}\sum_{i = 1}^n\nabla \loss_i(\xx_i^{(t)})}^2_2 + \frac{d \sigmacdp^2}{n} \\
	& \leq  \frac{\sigma_\star^2}{n}  + \frac{M}{n^2} \sum_{i=1}^n \norm{\nabla \loss(\xx_i^{(t)})}_2^2 + \norm{\frac{1}{n}\sum_{i = 1}^n\nabla \loss_i(\xx_i^{(t)})}^2_2 + \frac{d \sigmacdp^2}{n}\\
 & \leq  \frac{\sigma_\star^2}{n}  + \frac{2M}{n^2} \sum_{i=1}^n \norm{\nabla \loss(\xx_i^{(t)}) - \nabla \loss(\bar{\xx}^{(t)})}_2^2 + \frac{2M}{n}  \norm{\nabla \loss(\bar{\xx}^{(t)})}_2^2 +\norm{\frac{1}{n}\sum_{i = 1}^n\nabla \loss_i(\xx_i^{(t)})}^2_2 + \frac{d \sigmacdp^2}{n}\\
 & \leq  \frac{\sigma_\star^2}{n}  + \frac{2ML^2}{n^2} \sum_{i=1}^n \norm{\xx_i^{(t)} - \bar{\xx}^{(t)}}_2^2 + \frac{2M}{n}  \norm{\nabla \loss(\bar{\xx}^{(t)})}_2^2 +\norm{\frac{1}{n}\sum_{i = 1}^n\nabla \loss_i(\xx_i^{(t)})}^2_2 + \frac{d \sigmacdp^2}{n}.
	\end{align*}
    Combining the bounds on $A$ and $B$ in~\eqref{eq:descent2}, we obtain
    \begin{align*}
	\EE{t}{\loss(\bar{\xx}^{(t + 1)})} 
	& \leq \loss(\bar{\xx}^{(t)}) 
    -\frac{\eta_t}{2} \norm{\nabla \loss(\bar{\xx}^{(t)})}_2^2 
    - \frac{\eta_t}{2}\norm{\frac{1}{n}\sum_{i = 1}^n  \nabla \loss_i(\xx_i^{(t)})}_2^2
    + \frac{\eta_t L^2}{2} \frac{1}{n}\sum_{i=1}^n \norm{\xx_i^{(t)} - \bar{\xx}^{(t)}}_2^2\\
    &\quad+ \frac{L}{2} \eta_t^2 \left[\frac{\sigma_\star^2}{n}  + \frac{2ML^2}{n^2} \sum_{i=1}^n \norm{\xx_i^{(t)} - \bar{\xx}^{(t)}}_2^2 + \frac{2M}{n}  \norm{\nabla \loss(\bar{\xx}^{(t)})}_2^2 +\norm{\frac{1}{n}\sum_{i = 1}^n\nabla \loss_i(\xx_i^{(t)})}^2_2 + \frac{d \sigmacdp^2}{n}\right]\\
    & \leq \loss(\bar{\xx}^{(t)}) 
    -\frac{\eta_t}{2}(1-\frac{2ML}{n}\eta_t) \norm{\nabla \loss(\bar{\xx}^{(t)})}_2^2 
    - \frac{\eta_t}{2}(1-L\eta_t)\norm{\frac{1}{n}\sum_{i = 1}^n  \nabla \loss_i(\xx_i^{(t)})}_2^2\\
    &\quad+ \frac{\eta_t L^2}{2}(1+\frac{2ML}{n}\eta_t) \frac{1}{n}\sum_{i=1}^n \norm{\xx_i^{(t)} - \bar{\xx}^{(t)}}_2^2
    +\frac{L\eta_t^2}{2} \frac{\sigma_\star^2 + d \sigmacdp^2}{n}.
	\end{align*}
 
	By using $\eta_t \leq \frac{1}{2L}\min{\{1, \frac{n}{2M}\}}$ and taking total expectations, we conclude:	
 \begin{align*}
	\E{\left[\loss(\bar{\xx}^{(t + 1)}) - \loss(\bar{\xx}^{(t)})\right]} 
    & \leq -\frac{\eta_t}{4} \E\norm{\nabla \loss(\bar{\xx}^{(t)})}_2^2  + \frac{3\eta_t L^2}{4} \frac{1}{n}\sum_{i=1}^n \E\norm{\xx_i^{(t)} - \bar{\xx}^{(t)}}_2^2
    +\frac{L\eta_t^2}{2} \frac{\sigma_\star^2 + d \sigmacdp^2}{n}\\
    & = -\frac{\eta_t}{4} \E\norm{\nabla \loss(\bar{\xx}^{(t)})}_2^2  + \frac{3\eta_t L^2}{4} \Xi_t
    +\frac{L\eta_t^2}{2} \frac{\sigma_\star^2 + d \sigmacdp^2}{n}.
	\end{align*}
\end{proof}

\noisereduce*
\begin{proof}
Let $\cG = ([n], \cE)$ be an arbitrary undirected graph, and $\mW \in \R^{n \times n}$ be an arbitrary matrix (not necessarily a mixing matrix nor dependent upon $\cG$).
First, we prove that for every $j \in [n]$, we have
\begin{align}
    \mN^{(t)} \mW_j = \frac{1}{2} \sum_{i, k = 1}^{n} (\mW_{ij} - \mW_{kj}) \1_{k \in \cN_i} \vv_{ik}^{(t)},
    \label{lem5eq0}
\end{align}
where $\mW_j \in \R^n$ denotes the $j$-th column of $\mW$.
Indeed, we have
\begin{align}
    \mN^{(t)} \mW_j &= \sum_{i=1}^{n} \mW_{ij} N_i^{(t)}
    = \sum_{i=1}^{n} \sum_{k \in \cN_i} \mW_{ij} \vv_{ik}^{(t)}
    = \sum_{i,k=1}^{n} \mW_{ij} \1_{k \in \cN_i} \vv_{ik}^{(t)} \label{lem5eq1} \\
    &= \sum_{i,k=1}^{n} \mW_{ij} \1_{i \in \cN_k} \vv_{ik}^{(t)}
    = - \sum_{i,k=1}^{n} \mW_{ij} \1_{i \in \cN_k} \vv_{ki}^{(t)}
    = - \sum_{i,k=1}^{n} \mW_{kj} \1_{k \in \cN_i} \vv_{ik}^{(t)}, \label{lem5eq2}
\end{align}
where the last three equalities were successively obtained by using the facts that $\cG$ is undirected so $\1_{i \in \cN_k} = \1_{k \in \cN_i}$, that $\vv_{ik}^{(t)} = -\vv_{ik}^{(t)}, \forall i,k \in [n]$, and exchanging symbols $i,k$ in the double summation.
Thus, averaging equalities \eqref{lem5eq1} and \eqref{lem5eq2} proves Equation~\eqref{lem5eq0}.

Now, using Equation~\eqref{lem5eq0}, we can write
\begin{align*}
    \EE{}{\norm{\mN^{(t)} \mW}_F^2} &= \sum_{j=1}^n \EE{}{\norm{\mN^{(t)} \mW_j}^2} = \frac{1}{4} \sum_{j=1}^n \EE{}{\norm{\sum_{i, k = 1}^{n} (\mW_{ij} - \mW_{kj}) \1_{k \in \cN_i} \vv_{ik}^{(t)}}^2}\\
    & = \frac{1}{4} \sum_{j=1}^n \EE{}{ \norm{\sum_{\substack{i, k = 1 \\ i < k}}^{n} \left[ (\mW_{ij} - \mW_{kj}) \1_{k \in \cN_i} \vv_{ik}^{(t)} + (\mW_{kj} - \mW_{ij}) \1_{i \in \cN_k} \vv_{ki}^{(t)}\right]}^2}\\
    & {=} \frac{1}{4} \sum_{j=1}^n \EE{}{ \norm{2 \sum_{\substack{i, k = 1 \\ i < k}}^{n} (\mW_{ij} - \mW_{kj}) \1_{k \in \cN_i} \vv_{ik}^{(t)} }^2}
    = \sum_{j=1}^n \sum_{\substack{i, k = 1 \\ i < k}}^{n} (\mW_{ij} - \mW_{kj})^2 \1_{k \in \cN_i}  \EE{}{ \norm{\vv_{ik}^{(t)} }^2}\\
    & = \sum_{j=1}^n \sum_{\substack{i, k = 1 \\ i < k}}^{n} (\mW_{ij} - \mW_{kj})^2 \1_{k \in \cN_i} d \sigmacor^2
    = \frac{1}{2} \sum_{\substack{i,j, k = 1}}^{n} (\mW_{ij} - \mW_{kj})^2 \1_{k \in \cN_i} d \sigmacor^2\\
    & = \frac{1}{2} \sum_{i, k=1}^n  \norm{\mW_{i} - \mW_{k}}^2 \1_{k \in \cN_i}
     d \sigmacor^2,
\end{align*}
where in the fourth equality we used that $\vv_{ki}^{(t)} = - \vv_{ik}^{(t)}$ and that $\1_{i \in \cN_k} = \1_{k \in \cN_i}$, on the fifth equality we used that $\vv_{ik}^{(t)}$ are independent for $i < k$, and on the sixth equality that $\EE{}{\norm{\vv_{il}^{(t)}}^2} = d \sigmacor^2$.

Also, taking $\mW = \mathbf{I}_n$ in the equation above (which holds for arbitrary $\mW$), we have $\norm{\mW_{i} - \mW_{k}}^2 = 2 \cdot \1_{k \neq i}$, and thus
\begin{align*}
    \EE{}{\norm{\mN^{(t)}}_F^2}
    = \sum_{i, k=1}^n \1_{k \in \cN_i}
     d \sigmacor^2 = 2 \card{\cE} d \sigmacor^2.
\end{align*}
The last two equations directly lead to the main result of the lemma.

Now, denote by $k_{\mathrm{min}} \geq 1$ the minimal degree of $\cG$ and assume that $\mW_{ij}= \tfrac{\1_{j \in \cN_i}}{\mathrm{deg}(i)+1}, \forall i,j\in [n]$, where $\mathrm{deg}(i) = \card{\cN_i}$ is the degree of user $i$ in the graph. Thus, we have $\|\mW_i\|^2 = \tfrac{\mathrm{deg}(i)}{(\mathrm{deg}(i)+1)^2}$. Using Jensen's inequality, we have
\begin{align*}
    \sum_{i, k=1}^n  \norm{\mW_{i} - \mW_{k}}^2 \1_{k \in \cN_i} 
    &\leq 2\sum_{i, k=1}^n  \left(\norm{\mW_{i}} + \norm{\mW_{k}}^2\right) \1_{k \in \cN_i}
    = 4\sum_{i, k=1}^n \norm{\mW_{i}}^2\1_{k \in \cN_i}\\
    &= 4\sum_{i, k=1}^n \frac{\mathrm{deg}(i)}{(\mathrm{deg}(i)+1)^2}\1_{k \in \cN_i}
    = 4\sum_{i=1}^n \frac{\mathrm{deg}(i)^2}{(\mathrm{deg}(i)+1)^2}.
\end{align*}
On the other hand, we have $2\sum_{i,k=1}^n \1_{k \in \cN_i} = 2\sum_{i=1}^n \mathrm{deg}(i)$, so that
\begin{align*}
    \hg{\mW} = \frac{\sum_{i, k = 1}^{n} \norm{\mW_i - \mW_k}^2 \1_{k \in \cN_i}}{2\sum_{i, k = 1}^{n} \1_{k \in \cN_i}} \leq 2\frac{\sum_{i=1}^n \frac{\mathrm{deg}(i)^2}{(\mathrm{deg}(i)+1)^2}}{\sum_{i=1}^n \mathrm{deg}(i)} \leq 2 \max_{i \in [n]} \frac{\mathrm{deg}(i)}{(\mathrm{deg}(i)+1)^2} \leq \frac{2k_{\mathrm{min}}}{(k_{\mathrm{min}}+1)^2} \leq \frac{2}{k_{\mathrm{min}}}.
\end{align*}
This concludes the second statement of the lemma.
\end{proof}

\consensus*
\begin{proof}
Let assumptions~\ref{a:lsmooth}, \ref{a:opt_nc}, and \ref{a:avg_distrib} hold.
Also, assume that stepsizes verify $\eta_t \leq  \frac{p}{96\sqrt{6}\tau  L}$ for each iteration $t$.
Denote
\begin{equation*}
\partial \ell(\mX^{(t)}) := \left[\nabla \loss_1(\xx_{1}^{(t)}), \dots,  \nabla \loss_n(\xx_{n}^{(t)})\right]  \in \R^{d\times n}.
\end{equation*}
We first write
\begin{align*}
    \mX^{(t+1)} - \bar \mX^{(t+1)}
    &= \mX^{(t+\tfrac{1}{2})}\mW - \mX^{(t+\tfrac{1}{2})}\frac{\1\1^\top}{n} 
    = \mX^{(t+\tfrac{1}{2})}(\mW - \frac{\1\1^\top}{n})\\
    &= \left[\mX^{(t)} - \eta_t\left(\partial \ell(\mX^{(t)}, \xxi^{(t)}) + \mN^{(t)} + \bar{\mN}^{(t)} \right) \right](\mW - \frac{\1\1^\top}{n})\\
    &= \mX^{(t)}(\mW - \frac{\1\1^\top}{n}) - \eta_t \left(\partial \ell(\mX^{(t)}, \xxi^{(t)}) + \mN^{(t)} + \bar{\mN}^{(t)} \right)(\mW - \frac{\1\1^\top}{n})\\
    &= (\mX^{(t)} - \eta_t \partial \ell(\mX^{(t)}))(\mW - \frac{\1\1^\top}{n}) \\
    & \quad - \eta_t \left(\partial \ell(\mX^{(t)}, \xxi^{(t)}) - \partial \ell(\mX^{(t)}) + \mN^{(t)} + \bar{\mN}^{(t)} \right)(\mW - \frac{\1\1^\top}{n}).
\end{align*}
By independence, taking squared Frobenius norms and total expectations yields
\begin{align}
    n \Xi_{t+1} = \E{\norm{\mX^{(t+1)} - \bar \mX^{(t+1)}}_F^2}
    &= \E{\norm{(\mX^{(t)} - \eta_t \partial \ell(\mX^{(t)}))(\mW - \frac{\1\1^\top}{n})}_F^2} \nonumber\\
    & \quad + \eta_t^2 \E{\norm{\left(\partial \ell(\mX^{(t)}, \xxi^{(t)}) - \partial \ell(\mX^{(t)}) + \mN^{(t)} + \bar{\mN}^{(t)} \right)(\mW - \frac{\1\1^\top}{n})}_F^2}.
    \label{eq:consensus1}
\end{align}
The first term on the RHS of \eqref{eq:consensus1} can be bounded, by first using Assumption~\ref{a:avg_distrib} and then Young's inequality, as follows:
\begin{align*}
    &\E{\norm{(\mX^{(t)} - \eta_t \partial \ell(\mX^{(t)}))(\mW - \frac{\1\1^\top}{n})}_F^2}
    \leq (1-p) \E{\norm{\mX^{(t)} - \eta_t \partial \ell(\mX^{(t)}) - \bar \mX^{(t)} + \eta_t \partial \ell(\mX^{(t)}) \frac{\1 \1^\top}{n}}_F^2}\\
    &= (1-p) \E{\norm{\mX^{(t)} - \bar \mX^{(t)} - \eta_t( \partial \ell(\mX^{(t)}) - \partial \ell(\mX^{(t)}) \frac{\1 \1^\top}{n})}_F^2}\\
    &\leq (1-p)(1+\frac{p}{3(1-p)}) \E{\norm{\mX^{(t)} - \bar \mX^{(t)}}_F^2}
    + (1-p)(1+\frac{3(1-p)}{p}) \eta_t^2 \E{\norm{\partial \ell(\mX^{(t)}) - \partial \ell(\mX^{(t)}) \frac{\1 \1^\top}{n}}_F^2}\\
    &= (1-\frac{2p}{3}) n \Xi_t
    + \frac{(1-p)(3-2p)}{p} \eta_t^2 \E{\norm{\partial \ell(\mX^{(t)}) - \partial \ell(\mX^{(t)}) \frac{\1 \1^\top}{n}}_F^2}\\
    &\leq (1-\frac{2p}{3}) n \Xi_t
    + \frac{3(1-p)}{p} \eta_t^2 \E{\norm{\partial \ell(\mX^{(t)})}_F^2},
\end{align*}
where the last inequality is due to $p \geq 0$ and also that for any $A\in \R^{d\times n}$, $B\in \R^{n\times n}$, we have $\norm{AB}_F \leq \norm{A}_F \norm{B}_2$, along with the fact that $\norm{\mI_{n} - \frac{\1 \1^\top}{n}}_2 =1$.

The second term on the RHS of \eqref{eq:consensus1} can be bounded, using independence, as follows:
\begin{align*}
    &\E{\norm{\left(\partial \ell(\mX^{(t)}, \xxi^{(t)}) - \partial \ell(\mX^{(t)}) + \mN^{(t)} + \bar{\mN}^{(t)} \right)(\mW - \frac{\1\1^\top}{n})}_F^2}=\\
    &\quad \E{\norm{\left(\partial \ell(\mX^{(t)}, \xxi^{(t)}) - \partial \ell(\mX^{(t)}) \right)(\mW - \frac{\1\1^\top}{n})}_F^2}
    + \E{\norm{\mN^{(t)}(\mW - \frac{\1\1^\top}{n})}_F^2}
    + \E{\norm{\bar{\mN}^{(t)}(\mW - \frac{\1\1^\top}{n})}_F^2}.
\end{align*}
We note that since $\bar{\mN}^{(t)}$ is a matrix of $d \times n$ i.i.d. Gaussian variables of variance $\sigmacdp^2$, we have $\E{\norm{\bar{\mN}^{(t)}(\mW - \frac{\1\1^\top}{n})}_F^2} = \norm{\mW - \frac{\1\1^\top}{n}}_F^2 d n \sigmacdp^2$.
Moreover, using the fact that the sum of correlated noise terms is zero and then Lemma~\ref{lem:pairwisenoisereduce}, we have $\E{\norm{\mN^{(t)}(\mW - \frac{\1\1^\top}{n})}_F^2} = \E{\norm{\mN^{(t)}\mW}_F^2} = 2\hg{\mW} |\cE| d \sigmacor^2$. Plugging these last two results above, and then using Assumption~\ref{a:avg_distrib}, yields:
\begin{align*}
    &\E{\norm{\left(\partial \ell(\mX^{(t)}, \xxi^{(t)}) - \partial \ell(\mX^{(t)}) + \mN^{(t)} + \bar{\mN}^{(t)} \right)(\mW - \frac{\1\1^\top}{n})}_F^2}=\\
    &\quad \E{\norm{\left(\partial \ell(\mX^{(t)}, \xxi^{(t)}) - \partial \ell(\mX^{(t)}) \right)(\mW - \frac{\1\1^\top}{n})}_F^2}
    + 2\hg{\mW} |\cE| d \sigmacor^2
    + \norm{\mW - \frac{\1\1^\top}{n}}_F^2 d n \sigmacdp^2\\
    &\leq (1-p)\E{\norm{\partial \ell(\mX^{(t)}, \xxi^{(t)}) - \partial \ell(\mX^{(t)}) - \nabla \loss(\mX^{(t)},\xxi^{(t)})+\nabla\loss(\mX^{(t)})}_F^2}
    + 2\hg{\mW} |\cE| d \sigmacor^2
    + \norm{\mW - \frac{\1\1^\top}{n}}_F^2 d n \sigmacdp^2\\
    &\leq (1-p)\E{\norm{\partial \ell(\mX^{(t)}, \xxi^{(t)}) - \partial \ell(\mX^{(t)})}_F^2}
    + 2\hg{\mW} |\cE| d \sigmacor^2
    + \norm{\mW - \frac{\1\1^\top}{n}}_F^2 d n \sigmacdp^2.
\end{align*}
Reporting the previous bounds back in~\eqref{eq:consensus1} gives
\begin{align*}
    n \Xi_{t+1}
    &\leq (1-\frac{2p}{3}) n \Xi_t
    + \frac{3(1-p)}{p} \eta_t^2 \E{\norm{\partial \ell(\mX^{(t)})}_F^2} + \eta_t^2\Big((1-p)\E{\norm{\partial \ell(\mX^{(t)}, \xxi^{(t)}) - \partial \ell(\mX^{(t)})}_F^2}\\
    &\quad + 2\hg{\mW} |\cE| d \sigmacor^2
    + \norm{\mW - \frac{\1\1^\top}{n}}_F^2 d n \sigmacdp^2\Big).
\end{align*}
Rearranging and dividing by $n$ yields
\begin{align}
    \Xi_{t+1}
    &\leq (1-\frac{2p}{3}) \Xi_t
    +  \frac{\eta_t^2}{n} \Big[ \frac{3(1-p)}{p}\E{\norm{\partial \ell(\mX^{(t)})}_F^2} + (1-p)\E{\norm{\partial \ell(\mX^{(t)}, \xxi^{(t)}) - \partial \ell(\mX^{(t)})}_F^2}\nonumber\\
    &\quad + \frac{2\hg{\mW} |\cE| d \sigmacor^2}{n}
    + \norm{\mW - \frac{\1\1^\top}{n}}_F^2 d \sigmacdp^2\Big].
    \label{eq:consensus2}
\end{align}
On the one hand, by using assumptions~\ref{a:lsmooth} and~\ref{a:opt_nc} and Jensen's inequality, we have
\begin{align*}
    \E{\norm{\partial \ell(\mX^{(t)})}_F^2}
    &=\sum_{i=1}^n \E\norm{\nabla \loss_i(\xx_i^{(t)})}_2^2
    \leq 2\sum_{i=1}^n \E\norm{\nabla \loss_i(\xx_i^{(t)}) - \nabla \loss_i(\bar{\xx}^{(t)})}_2^2 + 2\sum_{i=1}^n\E\norm{\nabla \loss_i(\bar{\xx}^{(t)})}_2^2\\
    &\leq 2L^2 \sum_{i=1}^n \E\norm{\xx_i^{(t)}-\bar{\xx}^{(t)}}_2^2 + 2 n \zeta_\star^2 + 2nP \E\norm{\nabla \loss{(\bar{\xx}^{(t)})}}^2\\
    &= 2L^2 n \Xi_t + 2 n \zeta_\star^2 + 2nP \E\norm{\nabla \loss{(\bar{\xx}^{(t)})}}^2.
\end{align*}
On the other hand, by using assumptions~\ref{a:opt_nc} and~\ref{a:lsmooth} and Jensen's inequality, we obtain that
\begin{align*}
    \E{\norm{\partial \ell(\mX^{(t)}, \xxi^{(t)}) - \partial \ell(\mX^{(t)})}_F^2}
    &= \sum_{i=1}^n \E\norm{\nabla \ell(\xx_i^{(t)}, \xxi^{(t)}) - \nabla \loss_i(\xx_i^{(t)})}_2^2 \leq n \sigma_\star^2 + M \sum_{i=1}^n \E\norm{\nabla \loss(\xx_i^{(t)})}_2^2\\
    &\leq n \sigma_\star^2 + 2M \sum_{i=1}^n \E\norm{\nabla \loss(\xx_i^{(t)}) - \nabla \loss{(\bar{\xx}^{(t)})}}_2^2 + 2Mn \E\norm{\nabla \loss{(\bar{\xx}^{(t)})}}_2^2\\
    &\leq n \sigma_\star^2 + 2ML^2 \sum_{i=1}^n \E\norm{\xx_i^{(t)} - \bar{\xx}^{(t)}}_2^2 + 2Mn \E\norm{\nabla \loss{(\bar{\xx}^{(t)})}}_2^2\\
    &= n \sigma_\star^2 + 2ML^2 n \Xi_t + 2Mn \E\norm{\nabla \loss{(\bar{\xx}^{(t)})}}_2^2.
\end{align*}

By reporting the two bounds above back into~\eqref{eq:consensus2} and rearranging terms, and using $\eta_t \leq \frac{p}{L\sqrt{6(1-p)(3+pM)}}$ we obtain
\begin{align*}
    \Xi_{t+1}
    &\leq \left[1-\frac{2p}{3} + 2(1-p)L^2\eta_t^2(\frac{3}{p}+M)\right] \Xi_t
    +  2\eta_t^2(1-p)(\frac{3P}{p}+M) \norm{\nabla \loss(\bar{\xx}^{(t)})}_2^2\\
    &\quad + \eta_t^2 \left[6(1-p)\frac{\zeta_\star^2}{p} + (1-p)\sigma_\star^2 + \frac{2\hg{\mW} |\cE| d \sigmacor^2}{n}
    + \norm{\mW - \frac{\1\1^\top}{n}}_F^2 d \sigmacdp^2\right]\\
    &\leq (1-\frac{p}{2}) \Xi_t
    +  2\eta_t^2(1-p)(\frac{3P}{p}+M) \norm{\nabla \loss(\bar{\xx}^{(t)})}_2^2\\
    &\quad + \eta_t^2 \left[6(1-p)\frac{\zeta_\star^2}{p} + (1-p)\sigma_\star^2 + \frac{2\hg{\mW} |\cE| d \sigmacor^2}{n}
    + \norm{\mW - \frac{\1\1^\top}{n}}_F^2 d \sigmacdp^2\right].
\end{align*}
The above concludes the proof.
\end{proof}

\section{Privacy-utility Trade-off}
\label{sec:app-tradeoff}
In this section, we prove our main privacy result stated in Corollary~\ref{cor:final} and extend it to the general privacy adversaries discussed in Section~\ref{sec:problem}.
We first recall some useful facts around Rényi differential privacy (RDP)~\cite{mironov2017renyi}. 

\begin{lemma}[RDP Composition, \cite{mironov2017renyi}]
\label{lem:composition}
If a privacy mechanism $\cM_1$ that takes the dataset as input is $(\alpha,\varepsilon_1)$-RDP, and a privacy mechanism $\cM_2$ that takes the dataset and the output of $\cM_1$ as input is $(\alpha,\varepsilon_2)$-RDP, then their composition $\cM_2 \circ \cM_1$ is $(\alpha,\varepsilon_1+\varepsilon_2)$-RDP.
\end{lemma}

\begin{lemma}[RDP to DP conversion, \cite{mironov2017renyi}]
\label{lem:rdp-dp}
If a privacy mechanism $\cM$ is $(\alpha,\varepsilon)$-RDP, then $\cM$ is $(\varepsilon+\frac{\log{(1/\delta)}}{\alpha-1},\delta)$-DP for all $\delta \in (0,1)$.
\end{lemma}

\textbf{Proof of Corollary~\ref{cor:final}.}
For convenience, we restate Corollary~\ref{cor:final} below, whose proof is a special case of the extended privacy-utility trade-off result given next.

\cortradeoff*
\begin{proof}
    This result is a special case of Corollary~\ref{cor:final-general}, by taking $q=0$ for the external eavesdropper and $q=1$ for the honest-but-curious non-colluding users in the PL case, and omitting vanishing terms in $T$.
\end{proof}

\textbf{Extended privacy-utility trade-off.}
We now state and prove a general privacy-utility trade-off analysis of \alg to all considered adversaries in Section~\ref{sec:problem}, which includes collusion, as well as the non-convex case. 

\begin{corollary}
\label{cor:final-general}
Let the assumptions of theorems~\ref{th:privacy} and~\ref{thm:summary} hold and assume that $\cG$ is $(q+1)$-connected.
Let $\varepsilon>0, \delta \in (0,1)$ be such that $\varepsilon \leq \log{(1/\delta)}$.
Consider Algorithm~\ref{algo} with $\sigmacdp^2= \tfrac{32 C^2 T \log{(1/\delta)}}{(n-q)\varepsilon^2}$ and $\sigmacor^2= \tfrac{32 C^2 T\log{(1/\delta)}}{a_q{(\cG)}\varepsilon^2}$.
Denote $\loss_0 \coloneqq \loss(\bar{\xx}^{(0)}) - \loss_\star$.
Then, Algorithm~\ref{algo} satisfies $(\varepsilon, \delta)$-SecLDP and the following holds:

\begin{enumerate}
    \item Assume that $\loss$ is $\mu$-PL:
    \begin{align*}
        &\E\loss(\bar{\xx}^{(T)}) - \loss_\star=  \widetilde{\mathcal{O}}\Bigg(\frac{L C^2 d \log{(1/\delta)}}{\mu^2 n (n-q) \varepsilon^2}
        + \frac{L}{\mu^2 n T}\Big[\sigma_\star^2 +\frac{L C^2 d\log{(1/\delta)}}{\mu p \varepsilon^2}\Big(\frac{\hg{\mW} \card{\cE}}{a(\cG_{\cH})}
        + \frac{n}{(n-q)}\norm{\mW-\frac{\1\1^\top}{n}}_F^2\Big)\Big]\Bigg).
    \end{align*}
    \item In the general non-convex case:
    \begin{equation*}
    \hspace{-0.5cm}\frac{1}{T} \sum_{t=0}^{T-1} \E\norm{\nabla \loss(\bar{\xx}^{(t)})}_2^2
    =\mathcal{O}\left( \frac{C\sqrt{d \log{(1/\delta)}}}{\sqrt{n(n-q)}\varepsilon}+ \sqrt{\frac{L\loss_0\sigma_\star^2}{nT}}\right).
\end{equation*}
\end{enumerate}

\end{corollary}
\begin{proof}
Let the assumptions of theorems~\ref{th:privacy} and~\ref{thm:summary} hold.
Let $\varepsilon>0, \delta \in (0,1)$ be such that $\varepsilon \leq \log{(1/\delta)}$ and assume that $\cG$ is $(q+1)$-connected.
Consider Algorithm~\ref{algo} with $\sigmacdp^2= \frac{32 C^2 T \log{(1/\delta)}}{(n-q) \varepsilon^2}$ and $\sigmacor^2= \frac{32 C^2 T\log{(1/\delta)}}{a_q{(\cG)}\varepsilon^2}$. 
The latter quantity is well-defined as $\cG$ is $(q+1)$-connected and thus has positive algebraic connectivity after deleting any set of $q$ vertices~\citep{de2007old}.

\textbf{Privacy.}
We first show the privacy claim. Recall from Theorem~\ref{th:privacy} that each iteration of Algorithm~\ref{algo} satisfies $(\alpha, \alpha \varepsilon_{\text{step}})$-SecRDP against collusion at level $q$ for every $\alpha>1$ where
\begin{align}
    \varepsilon_{\text{step}} \leq 2 C^2 \left( \frac{1}{(n-q) \sigmacdp^2} + \frac{1}{a_q(\cG) \sigmacor^2}  \right).
\end{align}
Thus, following the composition property of RDP from Lemma~\ref{lem:rdp-dp}, the full Algorithm~\ref{algo} satisfies $(\alpha, T \alpha \varepsilon_{\text{step}})$-SecRDP for any $\alpha>1$. From Lemma~\ref{lem:rdp-dp}, we deduce that Algorithm~\ref{algo} satisfies $(\varepsilon'{(\alpha)}, \delta)$-SecLDP for any $\delta \in (0,1)$ and any $\alpha > 1$, where
\begin{align*}
    \varepsilon'{(\alpha)} = T \alpha \varepsilon_{\text{step}} + \frac{\log(1/\delta)}{\alpha-1} 
    \leq 2 \alpha C^2 T \left( \frac{1}{(n-q) \sigmacdp^2} + \frac{1}{a_q(\cG) \sigmacor^2}  \right)  + \frac{\log(1/\delta)}{\alpha-1}.
\end{align*}
Optimizing the above bound over $\alpha>1$ yields the solution $\alpha_\star = 1 + \tfrac{\sqrt{\log(1/\delta)}}{C\sqrt{2T \left( \tfrac{1}{(n-q) \sigmacdp^2} + \tfrac{1}{a_q(\cG) \sigmacor^2}  \right)}}$ which gives the bound
\begin{align*}
    \varepsilon_\star = \varepsilon'{(\alpha_\star)} 
    \leq 2 C^2 T \left( \frac{1}{(n-q) \sigmacdp^2} + \frac{1}{a_q(\cG) \sigmacor^2}  \right) + 2 C\sqrt{2 T \log{(1/\delta)} \left( \frac{1}{(n-q) \sigmacdp^2} + \frac{1}{a_q(\cG) \sigmacor^2}  \right)}.
\end{align*}
Now, recall that the choice of $\sigmacdp^2, \sigmacor^2$ implies that
\begin{align*}
    \frac{1}{(n-q)\sigmacdp^2} + \frac{1}{a_q{(\cG)} \sigmacor^2} = \frac{\varepsilon^2}{16 C^2 T \log{(1/\delta)}}.
\end{align*}
Therefore, using the assumption $\varepsilon \leq \log{(1/\delta)}$, Algorithm~\ref{algo} satisfies $(\varepsilon_\star, \delta)$-DP where
\begin{align*}
    \varepsilon_\star \leq \frac{\varepsilon^2}{8\log{(1/\delta)}} + \frac{\varepsilon}{\sqrt{2}} \leq \varepsilon.
\end{align*}
This concludes the proof of the privacy claim.

\textbf{Upper bound---PL case.}
Plugging the expressions of $\sigmacdp^2$ and $\sigmacor^2$ in the PL bound of Theorem~\ref{thm:summary} and rearranging terms yields
\begin{align*}
    &\E\loss(\bar{\xx}^{(T)}) - \loss_\star = \mathcal{O}\Bigg(\frac{L}{\mu^2 T}\frac{\sigma_\star^2+d\tfrac{C^2 T \log{(1/\delta)}}{(n-q) \varepsilon^2}}{n} + \frac{c^2L^2(\loss(\bar{\xx}^{(0)}) - \loss_\star)}{\mu^2p^2T^2} + \frac{cL^3}{\mu^2 p^2 T^2} \Xi_0\\
    &\quad + \frac{\ln{T}}{T^2}\frac{L^2}{\mu^3p}\left[(1-p)(\frac{\zeta_\star^2}{p} + \sigma_\star^2) + \frac{\hg{\mW} \card{\cE} d}{n} \tfrac{C^2 T \log{(1/\delta)}}{a(\cG_{\cH}) \varepsilon^2} + \norm{\mW-\frac{\1\1^\top}{n}}_F^2 d \tfrac{C^2 T \log{(1/\delta)}}{(n-q) \varepsilon^2}\right]\Bigg)\\
    &= \widetilde{\mathcal{O}}\Bigg( \frac{L C^2 d \log{(1/\delta)}}{\mu^2 n (n-q) \varepsilon^2} + \frac{L}{\mu^2 n T}\left[\sigma_\star^2 +\frac{L C^2 d\log{(1/\delta)}}{\mu p \varepsilon^2}\left(\frac{\hg{\mW} \card{\cE}}{a(\cG_{\cH})} + \frac{n}{n-q}\norm{\mW-\frac{\1\1^\top}{n}}_F^2\right)\right]\\
    &\qquad \quad + \frac{L^2}{\mu^3 p^2 T^2}\left[(1-p)(\zeta_\star^2 + p\sigma_\star^2) + c^2 \mu(\loss(\bar{\xx}^{(0)}) - \loss_\star) + c \mu L \Xi_0 \right] \Bigg).
\end{align*}

\textbf{Upper bound---Non-convex case.}
Plugging the expressions of $\sigmacdp^2$ and $\sigmacor^2$ in the non-convex bound of Theorem~\ref{thm:summary} and rearranging terms yields
\begin{align*}
    &\frac{1}{T} \sum_{t=0}^{T-1} \E\norm{\nabla \loss(\bar{\xx}^{(t)})}_2^2
    = \mathcal{O}\Bigg( \sqrt{\frac{L(\loss(\bar{\xx}^{(0)}) - \loss_\star)(\sigma_\star^2+d\sigmacdp^2)}{nT}} + \frac{cL(\loss(\bar{\xx}^{(0)}) - \loss_\star)}{pT} + \frac{L^2}{pT} \Xi_0\\
    &\quad + \frac{L(\loss(\bar{\xx}^{(0)}) - \loss_\star)n}{pT(\sigma_\star^2+d\sigmacdp^2)} \left(6(1-p)\frac{\zeta_\star^2}{p} + (1-p)\sigma_\star^2 + \frac{\hg{\mW} \card{\cE} d \sigmacor^2}{n} + \norm{\mW-\tfrac{\1\1^\top}{n}}_F^2 d \sigmacdp^2\right)\Bigg)\\
    &=\mathcal{O}\Bigg( \sqrt{\frac{L(\loss(\bar{\xx}^{(0)}) - \loss_\star)(\sigma_\star^2+d\tfrac{C^2 T \log{(1/\delta)}}{(n-q) \varepsilon^2})}{nT}} + \frac{cL(\loss(\bar{\xx}^{(0)}) - \loss_\star)}{pT} + \frac{L^2}{pT} \Xi_0\\
    &\quad + \frac{L(\loss(\bar{\xx}^{(0)}) - \loss_\star)n}{pT(\sigma_\star^2+d\tfrac{C^2 T \log{(1/\delta)}}{(n-q) \varepsilon^2})} \left((1-p)(\frac{\zeta_\star^2}{p} + \sigma_\star^2) + \frac{\hg{\mW} \card{\cE} d \tfrac{C^2 T \log{(1/\delta)}}{a(\cG_{\cH}) \varepsilon^2}}{n} + \norm{\mW-\tfrac{\1\1^\top}{n}}_F^2 d \tfrac{C^2 T \log{(1/\delta)}}{(n-q) \varepsilon^2}\right)\Bigg)\\
    &=\mathcal{O}\Bigg( \frac{C\sqrt{d \log{(1/\delta)}}}{\sqrt{n(n-q)}\varepsilon}+ \sqrt{\frac{L(\loss(\bar{\xx}^{(0)}) - \loss_\star)\sigma_\star^2}{nT}} + \frac{cL(\loss(\bar{\xx}^{(0)}) - \loss_\star)}{pT} + \frac{L^2}{pT} \Xi_0\\
    &\quad + \frac{L(\loss(\bar{\xx}^{(0)}) - \loss_\star)n}{pT} \left((1-p)(\frac{\zeta_\star^2}{p \sigma_\star^2} + 1) + \frac{(n-q)\hg{\mW} \card{\cE}}{n a_q{(\cG)}} + \norm{\mW-\tfrac{\1\1^\top}{n}}_F^2\right)\Bigg).
\end{align*}
We conclude by ignoring higher-order terms in $T$: in $\tfrac{1}{T^2}$ for the PL case and $\tfrac{1}{T}$ for the non-convex case.
\end{proof}
In the PL case, observe that our privacy-utility trade-off matches CDP whenever there is at most a constant fraction of colluding user, i.e., the level of collusion is $q = \mathcal{O}{(n)}$.
In the extreme scenario where almost all users are colluding, i.e., $n-q = \mathcal{O}{(1)}$, then the trade-off matches LDP only, which cannot be improved in general when $q=n-1$~\cite{duchi2018minimax}.
In the non-convex case, while it is not possible to discuss the tightness of our privacy-utility trade-off because lower bounds on the CDP trade-off are unknown, the error $\mathcal{O}{\left(\tfrac{\sqrt{d}}{n\varepsilon}\right)}$ matches the CDP baseline error without variance reduction~\cite{arora2022faster}. 

%% file: appendix_experiments.tex
\section{Detailed Experimental Setup}
\label{sec:app-expsetup}
In this section, we provide the full experimental setup of our empirical evaluation in Section~\ref{sec:exp}.

\textbf{Datasets.} We conduct our evaluation on three datasets: synthetic data for least-squares regression, \textit{a9a} LibSVM~\cite{chang2011libsvm} and MNIST~\cite{mnist}, that we distribute among $n = 16$ users, as explained in Section~\ref{sec:exp}.

\textbf{Privacy parameters.} We consider user-level privacy for the first two tasks, and example-level privacy for the last task. For all our experiments, we set the privacy parameter $\delta$ to $10^{-5}$, this ensures that $\delta \ll \tfrac{1}{nm} \leq \tfrac{1}{n}$. 
\subsection{Privacy Noise Parameters Search for \alg}
For a pre-specified SecLDP privacy budget $\varepsilon$, we would like to find a corresponding couple of privacy noises $(\sigmacdp,\sigmacor)$ to be used in \alg. However, Algorithm~\ref{algo:account} does the reverse process, i.e., it computes the per-step SecRDP budget, denoted $\varepsilon_{\mathrm{iter}}^{\mathrm{RDP}}$ here, given the privacy noise couple $(\sigmacdp,\sigmacor)$.
Moreover, it is straightforward to obtain the desired per-step RDP budget $\varepsilon_{\mathrm{iter}}^{\mathrm{RDP}}$ given the full DP budget $\varepsilon$ using composition and conversion properties of RDP~\cite{mironov2017renyi}.
Hence, we only need to search for $(\sigmacdp,\sigmacor)$, given a pre-specified $\varepsilon_{\mathrm{iter}}^{\mathrm{RDP}}$.
To do so, we fix $\sigmacdp$, and we look for the other parameter $\sigmacor$, using \emph{binary search}, since the function $\varepsilon_{\mathrm{iter}}^{\mathrm{RDP}}(\sigmacor)$ is monotonous (non-increasing), as shown in Figure~\ref{fig:epsilon-search-user}.
Specifically, we use the following steps in our search:
\begin{enumerate}
    \item Given the global (user-level) SecLDP privacy budget $\varepsilon$, we determine the per-step SecRDP privacy budget $\varepsilon_{\mathrm{iter}}^{\mathrm{RDP}}$ using the RDP composition and conversion properties
    \item We know that the uncorrelated noise variance $\sigmacdp$ is bounded between the privacy noise variance used for the baseline CDP algorithm $\tfrac{C \sqrt{2}}{\sqrt{n \varepsilon_{\mathrm{iter}}^{\mathrm{RDP}}}}$, and the one used for the LDP baseline, that is $\tfrac{C \sqrt{2}}{\sqrt{ \varepsilon_{\mathrm{iter}}^{\mathrm{RDP}}}}$. So we start by fixing $\sigmacdp$ in the interval $[\tfrac{C \sqrt{2}}{\sqrt{n \varepsilon_{\mathrm{iter}}^{\mathrm{RDP}}}}, \tfrac{C \sqrt{2}}{\sqrt{ \varepsilon_{\mathrm{iter}}^{\mathrm{RDP}}}}]$
    \item For every fixed $\sigmacdp$, we search for the corresponding $\sigmacor$ in a sufficiently large interval ($[1, 10^3]$ in our experiments) using binary search on the outputs of our SecRDP accountant (Algorithm~\ref{algo:account}). 
    \begin{figure}[H]
        \centering
        \includegraphics[width = 7 cm]{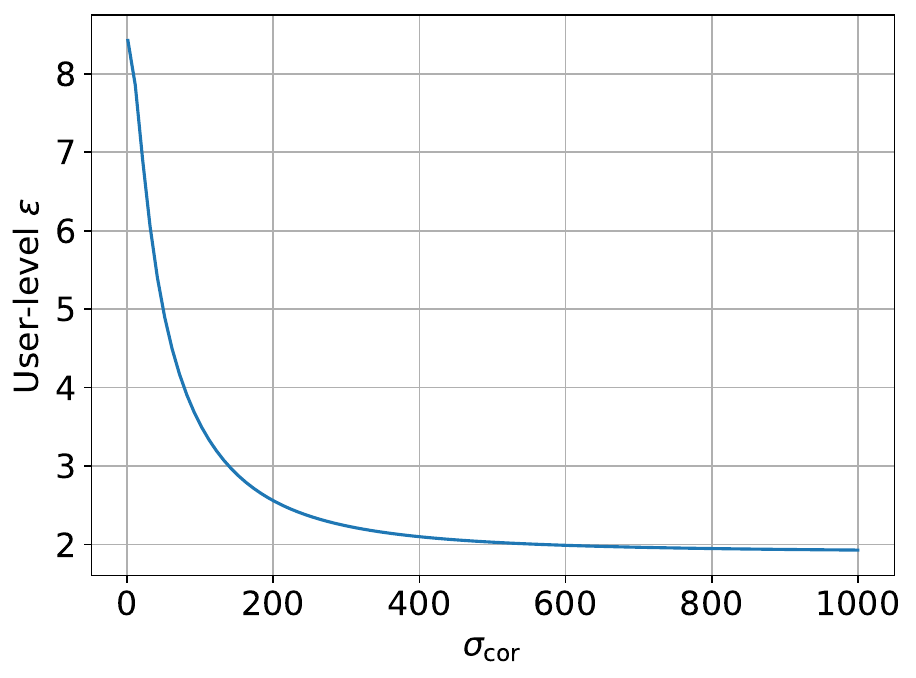}
        \caption{User-level SecLDP privacy budget $\varepsilon$, using Algorithm~\ref{algo:account}, as a function of $\sigmacor$ given a fixed $\sigmacdp$ in the center of the search interval, a total number of iterations $T = 1000$ and a clipping threshold $C = 1$.}
        \label{fig:epsilon-search-user}
    \end{figure}
\end{enumerate}

\textbf{Example-level privacy.} 
The procedure to get the privacy noise parameters is slightly different for example-level privacy. Indeed, we use RDP privacy amplification by subsampling~\cite{wang2019subsampled} after using Algorithm~\ref{algo:account}.
However, the RDP privacy amplification by subsampling does not have a closed-form expression, so we cannot directly get the desired per-step SecRDP budget from the full DP budget $\varepsilon$.
Therefore, we again fix $\sigmacdp$ in a grid, this time in $[\tfrac{C}{1000}, \tfrac{C}{20}]$, and we look for the other parameter $\sigmacor$ (this time in $[\tfrac{C}{2000}, \tfrac{C}{10}]$) using \emph{binary search}, since the function $\varepsilon_{\mathrm{iter}}^{\mathrm{RDP}}(\sigmacor)$ is monotonous (non-increasing), as shown in Figure~\ref{fig:epsilon-search-example}.

\begin{figure}[H]
    \centering
    \includegraphics[width = 7 cm]{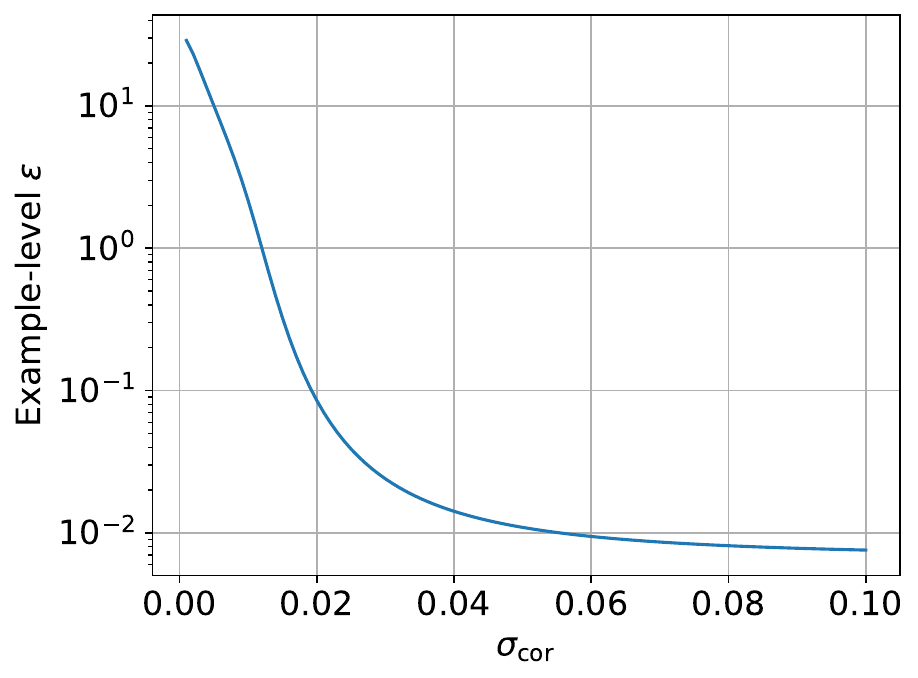}
    \caption{Example-level SecLDP privacy budget $\varepsilon$, using Algorithm~\ref{algo:account} and RDP amplification by subsampling~\cite{wang2019subsampled}, as function of $\sigmacor$ given a fixed $\sigmacdp = \frac{5C}{1000}$, a total number of iterations $T = 1000$, clipping threshold $C = 1$ and batch size $64$.}
    \label{fig:epsilon-search-example}
\end{figure}

\subsection{Hyperparameter Tuning}
For all considered tasks, we tune the hyperparameters of each algorithm individually, following the same steps, to obtain: the learning rate $\eta$, the clipping threshold $C$ and the noise parameters $\sigmacdp$ and $\sigmacor$. It is important to note that the couple of privacy noise parameters $(\sigmacdp, \sigmacor)$ is not unique: we can find many couples that yield the same SecRDP budget, which is also visible in the theoretical bound from Theorem~\ref{th:privacy}.
However, in the CDP and LDP baselines (D-SGD with uncorrelated privacy noise), they are determined uniquely by the RDP guarantee for the Gaussian mechanism~\cite{mironov2017renyi}.

For our tuning, we choose a grid of learning rates and clipping thresholds.
First, we simply evaluate the CDP and LDP baselines with the desired topology on all the learning rate and clipping couples $(\eta, C)$, and then we pick the best hyperparameter couple at the end.
For \alg, we do the same procedure for $(\eta, C)$. However, there are many possible noise couples $(\sigmacdp, \sigmacor)$ following the privacy noise search in the previous section, we choose three among them that yield the same privacy budget: the one with the lowest $\sigmacdp$ (first couple found by binary search), the largest $\sigmacdp$ (last couple) and the one in the middle. After evaluating these noises with every couple $(\eta, C)$, we choose at the end the best quadruplet of hyperparameters.